\documentclass{article}
\usepackage[latin9]{inputenc}
\usepackage{color}
\usepackage{colortbl}
\usepackage{verbatim}
\usepackage{multirow}
\usepackage{amsmath}
\usepackage{amsthm}
\usepackage{amssymb}
\usepackage{graphicx}
\usepackage{microtype}
\usepackage[bookmarksnumbered=true,
bookmarksopen=true,
backref=section,
colorlinks=true,
citecolor=blue,
linkcolor=blue,
anchorcolor=red,
unicode=true,
urlcolor=blue]{hyperref}

\usepackage{microtype}
\usepackage{graphicx}
\usepackage{subfigure}
\usepackage{booktabs} 

\DeclareTextSymbolDefault{\textquotedbl}{T1}
\providecommand{\tabularnewline}{\\}

\usepackage{prettyref}
\usepackage{soul}
\usepackage{url}

\newtheorem{thm}{Theorem}
\newtheorem{assum}{Assumption}
\newtheorem{lemma}{Lemma}
\newtheorem{definition}{Definition}
\newtheorem{corollary}{Corollary}

\newcommand{\RR}{\mathbb{R}}

\newcommand{\calL}{\mathcal{L}}
\newcommand{\eL}{\widehat{\mathcal{L}}_n}
\newcommand{\eR}{\widehat{\mathcal{R}}_n}

\newtheorem{remark}{Remark}

\usepackage{amsmath,amsfonts,bm}

\def\eqref#1{equation~\ref{#1}}

\def\1{\bm{1}}

\DeclareMathAlphabet{\mathsfit}{\encodingdefault}{\sfdefault}{m}{sl}
\SetMathAlphabet{\mathsfit}{bold}{\encodingdefault}{\sfdefault}{bx}{n}

\usepackage[accepted]{icml2020}

\usepackage{titlesec}
\titlespacing*{\section}{0pt}{0pt}{-4pt}
\titlespacing*{\subsection}{0pt}{0pt}{-4pt}
\titlespacing*{\subsubsection}{0pt}{0pt}{-4pt}

\usepackage{color}

\icmltitlerunning{DessiLBI: Exploring Structural Sparsity of Deep Networks via Differential Inclusion Paths}

\begin{document}
\abovedisplayskip=3pt
\belowdisplayskip=4pt
\abovedisplayshortskip=3pt
\belowdisplayshortskip=4pt

\twocolumn[
  \icmltitle{DessiLBI: Exploring Structural Sparsity of Deep Networks \\ via Differential Inclusion Paths}
  


%




\begin{icmlauthorlist}
\icmlauthor{Yanwei Fu}{fd}
\icmlauthor{Chen Liu}{fd}
\icmlauthor{Donghao Li}{fd,hk}
\icmlauthor{Xinwei Sun}{msra}
\icmlauthor{Jinshan Zeng}{hk,jx}
\icmlauthor{Yuan Yao}{hk}
\end{icmlauthorlist}

\icmlaffiliation{fd}{School of Data Science, and MOE Frontiers Center for Brain Science, Shanghai Key Lab of Intelligent Information Processing Fudan University. yanweifu@fudan.edu.cn}
\icmlaffiliation{hk}{Hong Kong University of Science and Technology}
\icmlaffiliation{msra}{Microsoft Research-Asia}
\icmlaffiliation{jx}{Jiangxi Normal University,}

\icmlcorrespondingauthor{Yuan Yao}{yuany@ust.hk}

\icmlkeywords{Deep Network Training, Linearized Bregman Iteration, Structural Sparsity, Differential Inclusion Paths}

\vskip 0.3in
]



\printAffiliationsAndNotice{}  

\begin{abstract}
	Over-parameterization is ubiquitous nowadays in training neural networks
	to benefit both optimization in seeking global optima and generalization
	in reducing prediction error. However, compressive networks are desired
	in many real world applications and direct training of small networks
	may be trapped in local optima. In this paper, instead of pruning
	or distilling over-parameterized models to compressive ones, we
	propose a new approach based on \emph{differential
		inclusions of inverse scale spaces}. Specifically, it generates a family of models
	from simple to complex ones that \textcolor{black}{couples a pair of parameters to simultaneously train over-parameterized deep models and structural sparsity} on weights of fully connected and convolutional layers. Such a differential inclusion scheme has a simple discretization, proposed as 
	\textbf{De}ep \textbf{s}tructurally \textbf{s}pl\textbf{i}tting \textbf{L}inearized \textbf{B}regman \textbf{I}teration (DessiLBI), whose
	global convergence analysis in deep learning is established that from any initializations, algorithmic
	iterations converge to a critical point of empirical risks. 
	Experimental evidence shows that DessiLBI achieve comparable and even better performance than 
the competitive optimizers in exploring the structural sparsity of several widely used backbones on the benchmark datasets. Remarkably, with \emph{early stopping},  DessiLBI unveils \textcolor{black}{``\emph{winning tickets}"} in early epochs: the effective sparse structure with comparable test accuracy to fully trained over-parameterized models. 
\end{abstract}
\vspace{-0.6cm}
\section{Introduction}

The expressive power of deep neural networks comes from the millions of parameters, which are optimized by Stochastic
Gradient Descent (SGD) \citep{bottou-2010} and variants
like Adam \citep{kingma2014adam}. Remarkably, model over-parameterization
helps both optimization and generalization. For optimization, over-parameterization
may simplify the landscape of empirical risks toward locating global
optima efficiently by gradient descent method \citep{Mei_pnas18,Mei_colt19,VenBanBru18,Allen-Zhu18,Du18}.
On the other hand, over-parameterization does not necessarily result
in a bad generalization or overfitting \citep{zhang16rethinking},
especially when some weight-size dependent complexities are controlled
\citep{Bartlett97,bartlett17spectral,sasha18size-independent,ZHY18,srebro19_iclr}.

However, compressive networks are desired in many real world applications,
\emph{e.g.} robotics, self-driving cars, and augmented reality. Despite
that $\ell_{1}$ regularization has been applied to deep learning 
to enforce the sparsity on weights toward compact, memory efficient
networks, it sacrifices some prediction performance \citep{l1_memory}.
This is because that the weights learned in neural networks are highly
correlated, and $\ell_{1}$ regularization on such weights violates
the incoherence or irrepresentable conditions needed for sparse model
selection \citep{DonHuo01,Tropp04,ZhaYu06}, leading to spurious selections
with poor generalization. On the other hand, $\ell_{2}$ regularization
is often utilized for correlated weights as some low-pass filtering,
sometimes in the form of weight decay \citep{LosHut19} or early stopping
\citep{YaoRosCap07,yuting17-nips}. Furthermore, group sparsity regularization
\citep{yuan2006model} has also been applied to neural networks, such
as finding optimal number of neuron groups \citep{group_spars_network}
and exerting good data locality with structured sparsity \citep{l12_norm,yoon2017combined}.

Yet, without the aid of over-parameterization, directly training a
compressive model architecture may meet the obstacle of being trapped
in local optima in contemporary experience. Alternatively, researchers
in practice typically start from training a big model using common
task datasets like ImageNet, and then prune or distill such big models
to small ones without sacrificing too much of the performance
 \citep{jaderberg2014speeding,han2015learning,li2016pruning,abbasi2017structural,arora2018stronger}.
In particular, a recent study \citep{lottery-iclr19} created the \emph{lottery
	ticket hypothesis} based on empirical observations: ``dense, randomly-initialized,
feed-forward networks contain subnetworks (winning tickets) that --
when trained in isolation -- reach test accuracy comparable to the
original network in a similar number of iterations\textquotedbl .
How to effectively reduce an over-parameterized model thus becomes
the key to compressive deep learning. Yet, \cite{rethinking_iclr} raised a question, \emph{is it necessary to fully train a dense, over-parameterized model before finding important structural sparsity}? 


This paper provides a novel answer by exploiting a dynamic
approach to deep learning with structural sparsity. We are able to establish a family
of neural networks, from simple to complex, by following regularization
paths as solutions of \emph{differential inclusions of inverse scale
	spaces}. Our key idea is to design some dynamics that simultaneously
exploit over-parameterized models and structural sparsity. To achieve
this goal, the original network parameters are lifted to a coupled
pair, with one \textit{weight set }$W$ of parameters following the
standard gradient descend to explore the over-parameterized model
space, while the other set of parameters $\Gamma$ learning structure sparsity
in an \emph{inverse scale space}. The large-scale important parameters
are learned at faster speed than small unimportant ones. The two sets of parameters are coupled in an $\ell_{2}$
regularization. This dynamics on highly non-convex (e.g. deep  models) setting enjoys a simple discretization, which is proposed as 
\textbf{De}ep \textbf{s}tructurally \textbf{s}pl\textbf{i}tting \textbf{L}inearized \textbf{B}regman \textbf{I}teration (DessiLBI)  with provable global convergence
guarantee  in this paper. Here, DessiLBI is a natural extension of SGD with structural sparsity exploration: DessiLBI reduces to the standard gradient descent method when the coupling regularization is weak, while reduces to a sparse mirror descent when the coupling is strong. 

Critically, DessiLBI enjoys a nice property that effective subnetworks can be rapidly learned via structural sparsity parameter $\Gamma$ by the iterative regularization path without fully training a dense network first. Particularly, support set of structural sparsity parameter $\Gamma$ learned in the early stage of this inverse scale space discloses important sparse subnetworks. Such architectures can be fine-tuned or retrained to achieve comparable test accuracy as the dense, over-parameterized networks. As a result, structural sparsity parameter $\Gamma$ may enable us to rapidly find ``winning tickets" in early training epochs for the ``lottery'' of identifying successful subnetworks that bear comparable test accuracy to the dense ones, confirmed empirically by experiments.

\vspace{-0.1cm}
\textbf{Contributions}. (1) DessiLBI is, for the first time, applied  
to explore the  structural sparsity of over-parameterized deep network via differential inclusion paths. 
 DessiLBI can be interpreted as the discretization of the dynamic approach of 
 differential inclusion paths in the inverse scale space. 
 (2) Global convergence of {DessiLBI} in such a nonconvex
optimization is established based on the Kurdyka-{\L}ojasiewicz framework, that the whole iterative sequence converges to a critical point of the empirical loss function from arbitrary initializations. (3) Stochastic variants of {DessiLBI}  demonstrate the comparable and even better performance than other 
	training algorithms on ResNet-18 in large scale training such as ImageNet-2012, among other datasets, together with additional structural sparsity in successful models for interpretability. (4) Structural sparsity parameters in {DessiLBI}  provide important information about subnetwork architecture with comparable or even better accuracies than dense models before and after retraining -- \emph{ {DessiLBI} with early stopping} can provide fast ``\emph{winning tickets}" without fully training dense, over-parameterized models.  



\vspace{-0.1cm}
\section{Preliminaries and Related Work}

\textbf{Mirror Descent Algorithm (MDA)} firstly proposed by \cite{NemYu83} to solve \textcolor{black}{constrained convex optimization $L^{\star}:=\min_{W\in K} \mathcal{L}(W)$} ($K$ is convex and compact), \textcolor{black}{can be understood as} a generalized projected gradient descent \cite{beck2003mirror} with respect to Bregman distance $B_{\Omega}(u,v) := \Omega(u) - \Omega(v) - \langle \nabla \Omega(v), u-v\rangle$ induced by a convex and differentiable function $\Omega(\cdot)$,
\begin{subequations}
\label{eq:mda} 
	\begin{align}
	Z_{k+1} & = Z_k - \alpha \nabla \mathcal{L}(W_k) \label{eq:mda-a}\\
	W_{k+1} & = \nabla \Omega^{\star}(Z_{k+1}) \label{eq:mda-b},
	\end{align}
\end{subequations}
where the conjugate function of $\Omega(\cdot)$ is 
$\Omega^{\star}(Z) := \sup_W \langle W,Z \rangle - \Omega(W)$.
Equation~(\ref{eq:mda}) optimizes 
 $W_{k+1} = \arg\min_z \langle z, \alpha\mathcal{L}(W_k) \rangle + B_{\Omega}(z,W_k)$ 
 \cite{nemirovski2012tutorial} in two steps:  Eq (\ref{eq:mda-a}) implements the gradient descent on $Z$ that is an element in dual space $Z_k=\nabla \Omega(W_k)$; and 
 Eq (\ref{eq:mda-b}) projects it back to the primal space. As step size $\alpha \to 0$, 
 MDA has the following limit dynamics as ordinary differential equation (ODE) \cite{NemYu83}:
\begin{subequations}
\label{eq:mda-ode} 
	\begin{align}
	\dot{Z}_t & = \alpha \nabla \mathcal{L}(W_t) \label{eq:mda-ode-a}\\
	W_{t} & = \nabla \Omega^{\star}(Z_{t}) \label{eq:mda-ode-b},
	\end{align}
\end{subequations} 
Convergence analysis with rates have been well studied for convex loss, that has been extended to stochastic version \cite{ghadimi2012optimal, nedic2014stochastic} and Nesterov acceleration scheme \cite{su2016differential,krichene2015accelerated}. For highly non-convex loss met in deep learning, \cite{azizan2019sto} established the convergence to global optima for \emph{overparameterized} networks, provided that (i) the initial point is close enough to the manifold of global optima; (ii) the $\Omega(\cdot)$ is strongly convex and differentiable. For non-differentiable $\Omega$ such as the Elastic Net penalty in compressed sensing and high dimensional statistics ($\Omega(W) = \Vert W \Vert_1 + \frac{1}{2\kappa} \Vert W \Vert_F^2$), Eq. (\ref{eq:mda}) is studied as the Linearized Bregman Iteration (LBI) in applied mathematics \cite{yin2008bregman,osher2016diff} that follows a discretized solution path of differential inclusions, to be discussed below. Such solution paths play a role of sparse regularization path where early stopped solutions are often better than the convergent ones when noise is present. In this paper, we investigate a varied form of LBI for the highly non-convex loss in deep learning models, exploiting the sparse paths, and establishing its convergence to a KKT point for \emph{general} networks from \emph{arbitrary initializations}.  
\vspace{-0.2cm}


\textbf{Linearized Bregman Iteration~(LBI)}, was proposed in \cite{osher2005iterative, yin2008bregman} that firstly studies Eq.~(\ref{eq:mda}) when $\Omega(W)$ involves $\ell_1$ or total variation non-differentiable penalties met in compressed sensing and image denoising. 
Beyond convergence for convex loss \citep{yin2008bregman,cai2009convergence}, Osher et al. (\citeyear{osher2016diff}) and Huang et al. (\citeyear{huang18_aistats}) particularly showed that LBI is a discretization of differential inclusion dynamics whose solutions generate iterative sparse regularization paths, and established the statistical model selection consistency for high-dimensional generalized linear models. Moreover, Huang et al. (\citeyear{huang16_nips,huang18_acha}) further improved this by proposing SplitLBI, incorporating into LBI a variable splitting strategy such that the restricted Hessian with respect to augmented variable ($\Gamma$ in Eq.~\ref{eq:slbi-iss}) is orthogonal. This can alleviate the multicollinearity problem when the features are highly correlated; and thus can relax the irrepresentable condition,\textit{ i.e.}, the necessary condition for Lasso to have model selection consistency \cite{Tropp04,ZhaYu06}. 
However, existing work on SplitLBI is restricted to convex problems in generalized linear modes. It remains unknown whether the algorithm can exploit the structural sparsity in highly non-convex deep networks. To fill in this gap, in this paper, we propose the deep Structural Splitting LBI that simultaneously explores the overparameterized networks and the structural sparsity of the weights of fully connected and convolutional layers in such networks, which enables us to generate an iterative solution path of deep models whose important sparse architectures are unveiled in early stopping. 
\vspace{-0.2cm}
 \textbf{Alternating Direction Method of Multipliers (ADMM)} which also adopted variable splitting strategy, breaks original complex loss into smaller pieces with each one can be easily solved iteratively~\cite{wahlberg2012admm, boyd2011distributed}. Equipped with the variable splitting term, \cite{he20121, wang2013online} and \cite{zeng2019convergence} established the convergence result of ADMM in convex, stochastic and non-convex setting, respectively. \cite{wang2014bregman} studied convergence analysis with respect to Bregman distance. Recently, \cite{franca2018admm} derived the limit ODE dynamics of ADMM for convergent analysis. However, one should distinguish the LBI dynamics from ADMM that LBI should be viewed as a discretization of differential inclusion of inverse scale space that generalizes a sparse regularization solution path from simple to complex models where early stopping helps find important sparse models; in contrast, the ADMM, as an optimization algorithm for a given objective function, focuses on convergent property of the iterations. 

\vspace{-0.2cm}
\section{Methodology}
\begin{figure*}
	\centering{}\includegraphics[scale=0.34]{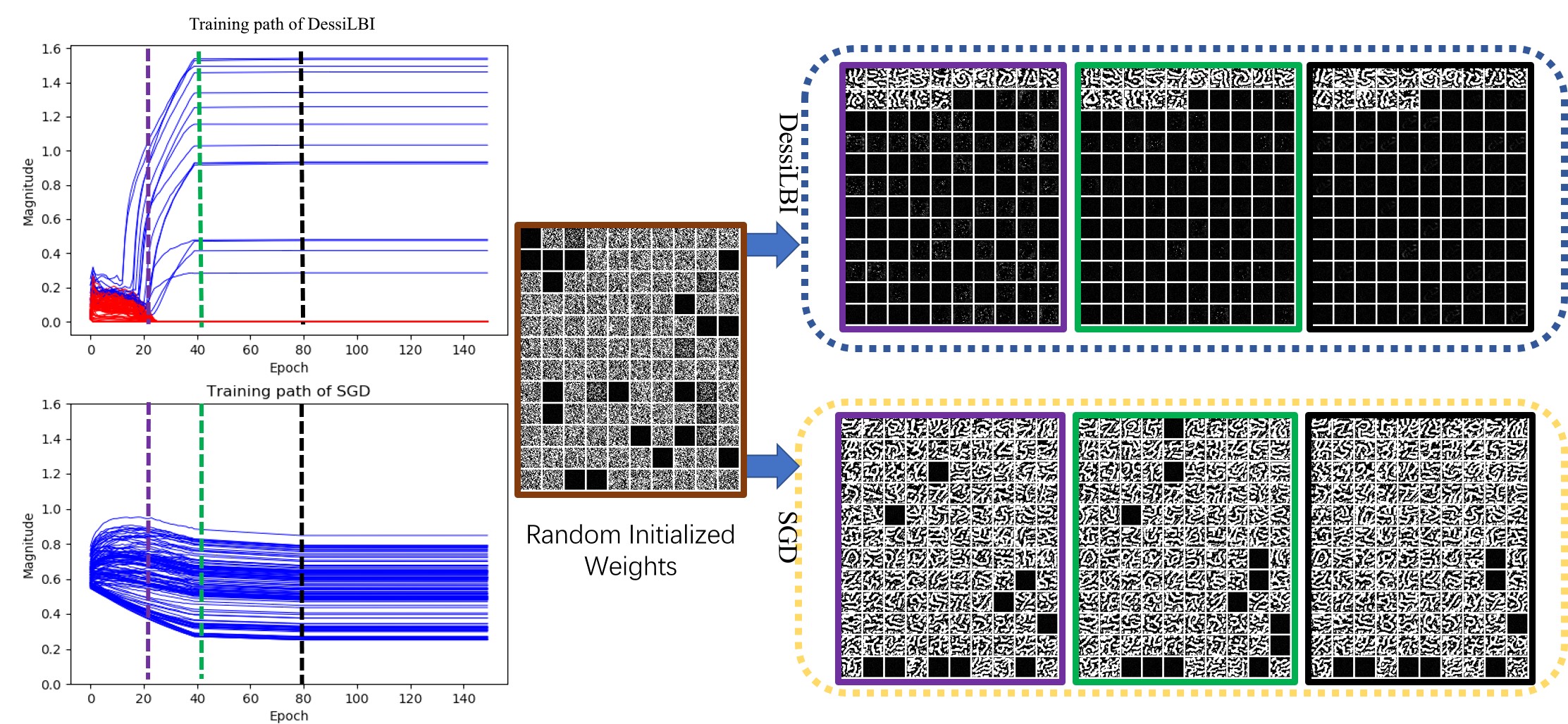}
\vspace{-0.15in}	
	\caption{{\small{}{}Visualization of solution path and filter patterns in
			the third convolution al layer (i.e., conv.c5) of LetNet-5, trained
			on MNIST. The left figure shows the magnitude changes for each filter
			of the models trained by DessiLBI  and SGD, where $x$-axis and $y$-axis
			indicate the training epochs, and filter magnitudes ($\ell_{2}$-norm),
			respectively. The DessiLBI  path of filters selected in the support
			of $\Gamma$ are drawn in blue color, while the red color curves represent
			the filters that are not important and outside the support of $\Gamma$.
			We visualize the corresponding learned filters by \citet{erhan2009visualizing}
			at 20 (blue), 40 (green), and 80 (black) epochs, which are shown in
			the right figure with the corresponding color bounding boxes, }\emph{\small{}{}i.e.}{\small{}{},
			blue, green, and black, respectively. It shows that our DessiLBI  enjoys
			a sparse selection of filters without sacrificing accuracy (see Table~\ref{table:supervised_imagenet}).}}
			\vspace{-0.15in}	
	\label{mnist_visualization} 
\end{figure*}

\label{sc:method}

Supervised learning learns  $\Phi_{W}:\mathbf{\mathcal{X}}\to\mathbf{\mathcal{Y}}$,
from input  $\mathcal{X}$ to output space $\mathcal{Y}$, with
a parameter $W$ such as weights in neural networks, by minimizing
certain loss functions on training samples $\eL(W)=\frac{1}{n}\sum_{i=1}^{n}\ell(y_{i},\Phi_{W}(x_{i}))$.
For example, a neural network of $l$-layer is defined as $\Phi_{W}(x)=\sigma_{l}\left(W^{l}\sigma_{l-1}\left(W^{l-1}\cdots\sigma_{1}\left(W^{1}x\right)\right)\right)$,
where $W=\{W^{i}\}_{i=1}^{l}$, $\sigma_{i}$ is the nonlinear activation
function of the $i$-th layer.

\textbf{Differential Inclusion of Inverse Scale Space.} Consider the
following dynamics, 
\begin{subequations}\label{eq:slbi-iss} 
	\begin{align}
	\frac{\dot{W_{t}}}{\kappa} & =-\nabla_{W}\bar{\calL}\left(W_{t},\Gamma_{t}\right)\label{eq:slbi-iss-show-a}\\
	\dot{V_{t}} & =-\nabla_{\Gamma}\bar{\mathcal{L}}\left(W_{t},\Gamma_{t}\right)\label{eq:slbi-iss-show-b}\\
	V_{t} & \in \partial\bar{\Omega}(\Gamma_{t}) \label{eq:slbi-iss-show-c}
	\end{align}
\end{subequations} 
where $V$ is a sub-gradient of $\bar{\Omega}(\Gamma):=\Omega_{\lambda}(\Gamma)+\frac{1}{2\kappa}\|\Gamma\|^{2}$
for some sparsity-enforced, often non-differentiable regularization $\Omega_\lambda(\Gamma)=\lambda \Omega_1(\Gamma)$ ($\lambda\in\RR_{+}$) such as
Lasso or group Lasso penalties for $\Omega_1(\Gamma)$, $\kappa>0$ is a damping parameter such
that the solution path is continuous, and the augmented loss function
is 
\begin{equation}
\bar{\calL}\left(W,\Gamma\right)=\eL\left(W\right)+\frac{1}{2\nu}\|W-\Gamma\|_F^2,\label{eq:sparse_loss}
\end{equation}
with $\nu>0$ controlling the gap admitted between $W$ and $\Gamma$. Compared to the original loss function $\eL\left(W\right)$, our loss $\bar{\calL}\left(W,\Gamma\right)$ additionally uses variable splitting strategy  by lifting the original neural network parameter $W$ to $(W,\Gamma)$ with $\Gamma$ modeling the structural sparsity of $W$. For simplicity, we assumed $\bar{\calL}$ is differentiable with respect to $W$ here, otherwise the gradient in Eq. (\ref{eq:slbi-iss-show-a}) is understood as subgradient and the equation becomes an inclusion.

Differential inclusion system (Eq.~\ref{eq:slbi-iss}) is a coupling of gradient descent on $W$ with non-convex loss and mirror descent (LBI) of $\Gamma$ (Eq.~\ref{eq:mda-ode}) with non-differentiable sparse penalty. It may explore dense over-parameterized models $W_t$ in the proximity of structural parameter $\Gamma_t$ with gradient descent, while $\Gamma_t$ records important sparse model structures. 


Specifically, the solution path of $\Gamma_{t}$ exhibits the following property in the separation of scales: starting at the zero, important parameters of large scale will be learned fast, popping up to be nonzeros early, while unimportant parameters of small scale will be learned slowly, appearing to be nonzeros late. In fact, taking $\Omega_\lambda(\Gamma)=\|\Gamma\|_{1}$ and $\kappa\to\infty$ for simplicity, $V_{t}$ as the subgradient of $\bar{\Omega}_t$, undergoes a gradient descent
flow before reaching the $\ell_{\infty}$-unit box, which implies that $\Gamma_{t}=0$ in this stage. The earlier a component in $V_{t}$ reaches the $\ell_{\infty}$-unit box, the earlier a corresponding component in $\Gamma_{t}$ becomes nonzero and rapidly evolves toward a critical point of $\bar{\calL}$ under gradient flow. On the other hand, the $W_{t}$ follows the gradient descent with a standard $\ell_{2}$-regularization. Therefore, $W_{t}$ closely follows dynamics of $\Gamma_{t}$ whose important parameters are selected.



Compared with directly enforcing a penalty function such as $\ell_{1}$ or $\ell_{2}$ regularization 
\begin{align}
\min_{W}\eR(W):=\eL\left(W\right)+ & \Omega_\lambda\left(W\right),\ \ \ \lambda\in\RR_{+}.\label{Eq:min-ERM}
\vspace{-0.15in}
\end{align}
dynamics Eq.~\ref{eq:slbi-iss} can relax the irrepresentable conditions for model selection by Lasso \cite{huang16_nips}, which can be violated for highly correlated weight parameters. The weight  $W$, instead of directly being imposed with
$\ell_{1}$-sparsity, adopts $\ell_{2}$-regularization in the proximity of the sparse path of $\Gamma$ that admits simultaneously exploring highly correlated parameters in over-parameterized models and sparse regularization.

The key insight lies in that differential inclusion of Eq.~\ref{eq:slbi-iss-show-c} drives the important features in $\Gamma_t$ that earlier reaches the $\ell_{\infty}$-unit box to be selected earlier. Hence, the importance of features is related to the ``time scale" of dynamic hitting time to the $\ell_{\infty}$ unit box, and such a time scale is inversely proportional to lasso regularization parameter $\lambda = 1/t$ \cite{osher2016diff}. Such a differential inclusion is firstly studied in \cite{BGOX06} with Total-Variation sparsity for image reconstruction, where important features in early dynamics are coarse-grained shapes with fine details appeared later. This is in contrast to wavelet scale space that coarse-grained features appear in large scale spaces, thus named ``inverse scale space''. In this paper, we shall see that Eq.~\ref{eq:slbi-iss} inherits such an inverse scale space property empirically even for the highly nonconvex neural network training. Figure~\ref{mnist_visualization} shows a LeNet trained on MNIST by the discretized dynamics, where important sparse filters are selected in early epochs while the popular SGD returns dense filters.

\textbf{Deep Structural Splitting Linearized Bregman Iteration.} 
Eq. \ref{eq:slbi-iss} admits
an extremely simple discrete approximation, using Euler forward
discretization of dynamics and called DessiLBI in the sequel: 
\begin{subequations} 
	\begin{align}
	& W_{k+1}=W_{k}-\kappa\alpha_{k}\cdot\nabla_{W}\bar{\mathcal{L}}\left(W_{k},\Gamma_{k}\right),\label{Eq:SLBI-iterate1}\\
	& V_{k+1}=V_{k}-\alpha_{k}\cdot\nabla_{\Gamma}\bar{\mathcal{L}}\left(W_{k},\Gamma_{k}\right),\label{Eq:SLBI-iterate2}\\
	& \Gamma_{k+1}=\kappa\cdot\mathrm{Prox}_{\Omega_\lambda}\left(V_{k+1}\right),\label{Eq:SLBI-iterate3}
	\end{align}
\end{subequations} where $V_{0}=\Gamma_{0}=0$, $W_{0}$ can be small
random numbers such as Gaussian initialization. For
some complex networks, it can be initialized as common setting. The
proximal map in Eq. (\ref{Eq:SLBI-iterate3}) that controls the sparsity
of $\Gamma$,
\begin{align}
\mathrm{Prox}_{\Omega_\lambda}(V)=\arg\min_{\Gamma}\ \left\{ \frac{1}{2}\|\Gamma-V\|_{2}^{2}+\Omega_\lambda\left(\Gamma\right)\right\} ,\label{Eq:prox-operator}
\end{align}
Such an iterative procedure 
returns a sequence of sparse networks from simple to complex ones whose global
convergence condition to be shown below,
 while solving Eq. (\ref{Eq:min-ERM})
at various levels of $\lambda$ might not be tractable, \textcolor{black}{especially} for over-parameterized networks.

Our DessiLBI explores structural sparsity in fully connected and convolutional layers, \textcolor{black}{which can be unified in framework of group lasso penalty}, $\Omega_1(\Gamma)=\sum_{g}\Vert \Gamma^{g}\Vert_{2}$,
where $\Vert \Gamma^{g}\Vert_{2}=\sqrt{\sum_{i=1}^{\mid \Gamma^{g}\mid}\left(\Gamma_{i}^{g}\right)^{2}}$ and 
$\left|\Gamma^{g}\right|$ is the number of weights in $\Gamma^{g}$. Thus Eq.~(\ref{Eq:SLBI-iterate3}) has a closed form solution $\Gamma^{g}=\kappa\cdot\max\left(0,1-1/\Vert V^{g}\Vert_{2}\right)V^{g}$.  Typically,


(1) For a convolutional layer, $\Gamma^{g}=\Gamma^{g}(c_{in},c_{out},\mathtt{size})$
denote the convolutional filters where $\mathtt{size}$ denotes the
kernel size and $c_{in}$ and $c_{out}$ denote the numbers of input
channels and output channels, respectively. When we regard each group as each convolutional filter, $g=c_{out}$; otherwise for weight sparsity, $g$ can be every element in the filter that reduces to the Lasso. \newline
(2) For a fully connected layer, $\Gamma=\Gamma(c_{in},c_{out})$
where $c_{in}$ and $c_{out}$ denote the numbers of inputs and outputs
of the fully connected layer. Each group $g$ corresponds to each
element $(i,j)$, and the group Lasso penalty degenerates to the Lasso
penalty.

In addtion, we can  take the group of incoming weights  $\Gamma^{g}=\Gamma^{g}(c_{in}, g)$ denoting the incoming weights of the $g$-th neuron of fc layers. This will be explored in future work.

\section{Global Convergence of DessiLBI }
\label{sc:theory}
\vspace{-0.02in}

We present a theorem that guarantees the \emph{global convergence}
of DessiLBI, \emph{i.e.} from any intialization, the DessiLBI  sequence
converges to a critical point of $\bar{\calL}$. Our treatment extends the block coordinate descent (BCD) studied in \cite{Zeng2019}, with a crucial difference being the mirror descent involved in DessiLBI. Instead of the splitting loss in BCD, a new Lyapunov function is developed here to meet the Kurdyka-{\L }ojasiewicz property \cite{Lojasiewicz-KL1963}. \cite{xin2018rvsm} studied convergence of variable splitting method for single hidden layer networks with Gaussian inputs.

Let $P:=(W,\Gamma)$. Following \cite{huang18_aistats}, the DessiLBI  algorithm in
Eq. (\ref{Eq:SLBI-iterate1}-\ref{Eq:SLBI-iterate3}) can be rewritten
as the following standard Linearized Bregman Iteration,
\begin{equation}
 P_{k+1}=\arg\min_{P}\left\{ \langle P-P_{k},\alpha\nabla\bar{\calL}(P_{k})\rangle+B_{\Psi}^{p_{k}}(P,P_{k})\right\} \label{Eq:SLBI-reformulation}
 \vspace{-0.1in}
\end{equation}
where 
\begin{align}
 \vspace{-0.1in}
\Psi(P) & =\Omega_\lambda(\Gamma)+\frac{1}{2\kappa}\|P\|_{2}^{2} \nonumber \\
& =\Omega_\lambda(\Gamma)+\frac{1}{2\kappa}\|W\|_{2}^{2}+\frac{1}{2\kappa}\|\Gamma\|_{2}^{2},\label{Eq:Phi-tilde}
\end{align}
$p_{k}\in\partial\Psi(P_{k})$, and $B_{\Psi}^{q}$ 
is the Bregman divergence associated with convex function $\Psi$,
defined by 
\begin{align}
B_{\Psi}^{q}(P,Q) & :=\Psi(P)-\Psi(Q)-\langle q,P-Q\rangle. \label{Eq:Bregman-divergence}
\end{align}
for some $q\in\partial\Psi(Q)$. Without loss of generality, consider $\lambda=1$ in the sequel. One can establish the global convergence of DessiLBI  under the following assumptions.

\begin{assum} 
\vspace{-0.02in}
\label{Assumption} Suppose that: (a) $\eL(W)=\frac{1}{n}\sum_{i=1}^{n}\ell(y_{i},\Phi_{W}(x_{i}))$ is
	continuous differentiable and $\nabla\eL$ is Lipschitz continuous
	with a positive constant $Lip$; (b)$\eL(W)$ has bounded level sets;
	(c) $\eL(W)$ is lower bounded (without loss of generality, we assume
	that the lower bound is $0$); (d) $\Omega$ is a proper lower semi-continuous
	convex function and has locally bounded subgradients, that is, for
	every compact set ${\cal S}\subset\mathbb{R}^{n}$, there exists a
	constant $C>0$ such that for all $\Gamma\in{\cal S}$ and all $g\in\partial\Omega(\Gamma)$,
	there holds $\|g\|\leq C$; and (e) the Lyapunov function 
	\begin{align}
	F(P,\tilde{g}):=\alpha\bar{\calL}(W,\Gamma)+B_{\Omega}^{\tilde{g}}(\Gamma,\tilde{\Gamma}),
	\label{Eq:Lyapunov-fun}
	\end{align}
	is a Kurdyka-{\L }ojasiewicz function on any bounded set, where
	$B_{\Omega}^{\tilde{g}}(\Gamma,\tilde{\Gamma}):=\Omega(\Gamma)-\Omega(\tilde{\Gamma})-\langle\tilde{g},\Gamma-\tilde{\Gamma}\rangle$,
	$\tilde{\Gamma}\in\partial\Omega^{*}(\tilde{g})$, and $\Omega^{*}$
	is the conjugate of $\Omega$ defined as 
	\begin{align*}
	\Omega^{*}(g):=\sup_{U\in\mathbb{R}^{n}}\{\langle U,g\rangle-\Omega(U)\}.
	\end{align*}
\vspace{-0.05in}
\end{assum}

\begin{remark} 
	\vspace{-0.05in}
Assumption \ref{Assumption} (a)-(c) are regular in
	the analysis of nonconvex algorithm (see, \cite{Attouch2013} for
	instance), while Assumption \ref{Assumption} (d) is also mild including
	all Lipschitz continuous convex function over a compact set. Some
	typical examples satisfying Assumption \ref{Assumption}(d) are the
	$\ell_{1}$ norm, group $\ell_{1}$ norm, and every continuously differentiable
	penalties. By Eq. (\ref{Eq:Lyapunov-fun}) and the definition of conjugate,
	the Lyapunov function $F$ can be rewritten as follows, 
	\begin{align}
	F(W,\Gamma,g)=\alpha\bar{\calL}(W,\Gamma)+\Omega(\Gamma)+\Omega^{*}(g)-\langle\Gamma,g\rangle.\label{Eq:Lyapunov-fun-conjugate}
	\vspace{-0.1in}
	\end{align}
	\vspace{-0.05in}
\end{remark}
Now we are ready to present the main theorem. 
\begin{thm}{[}Global Convergence of DessiLBI{]} \label{Thm:conv-SLBI}
 Suppose that Assumption
	\ref{Assumption} holds. Let $(W_{k},\Gamma_{k})$ be the sequence
	generated by DessiLBI  (Eq. (\ref{Eq:SLBI-iterate1}-\ref{Eq:SLBI-iterate3}))
	with a finite initialization. If 
	\begin{align*}
	0<\alpha_{k}=\alpha<\frac{2}{\kappa(Lip+\nu^{-1})},
	\vspace{-0.1in}
	\end{align*}
	then $(W_{k},\Gamma_{k})$ converges to a critical point of $\bar{\mathcal{L}}$ defined in Eq. (\ref{eq:sparse_loss}),
	and $\{W^{k}\}$ converges to a critical point of $\eL(W)$. 
\end{thm}
Applying to the neural networks, typical examples are summarized in
the following corollary.
\begin{corollary} \label{Corollary:DL} Let $\{W_{k},{\Gamma}_{k},g_{k}\}$
	be a sequence generated by DessiLBI (\ref{Eq:SLBI-reform-iter1}-\ref{Eq:SLBI-reform-iter3})
	for neural network training 
	where (a) $\ell$ is any smooth definable loss function, such as the
	square loss $(t^{2})$, exponential loss $(e^{t})$, logistic loss
	$\log(1+e^{-t})$, and cross-entropy loss; (b) $\sigma_{i}$ is any
	smooth definable activation, such as linear activation $(t)$, sigmoid
	$(\frac{1}{1+e^{-t}})$, hyperbolic tangent $(\frac{e^{t}-e^{-t}}{e^{t}+e^{-t}})$,
	and softplus ($\frac{1}{c}\log(1+e^{ct})$ for some $c>0$) as a smooth
	approximation of ReLU; (c) $\Omega$ is the group Lasso. 
	Then the sequence $\{W_{k}\}$ converges to a stationary point of
	$\eL(W)$ under the conditions of Theorem \ref{Thm:conv-SLBI}. \end{corollary}


\section{Experiments}

\label{sec:exp}

%

This section introduces some stochastic variants of {DessiLBI}, followed by four set of experiments revealing the insights of  {DessiLBI} exploring structural sparsity of deep networks.



\textbf{Batch {DessiLBI}}. To train networks on large datasets,
stochastic approximation of the gradients in DessiLBI  over the mini-batch
$(\mathbf{X},\mathbf{Y})_{{\mathrm{batch_{t}}}}$ is adopted to update
the parameter $W$, 
\begin{equation}
\widetilde{\nabla}_{W}^{t}=\nabla_{W}\eL\left(W\right)\mid{}_{(\mathbf{X},\mathbf{Y})_{\mathrm{batch}_{t}}}.\label{eq:sgd}
\end{equation}
\textbf{{DessiLBI}  with momentum (Mom)}. Inspired by the variants of SGD,
the momentum term can be also incorporated to the standard DessiLBI
that leads to the following updates of $W$ by replacing Eq (\ref{Eq:SLBI-iterate1})
with, 
\begin{subequations} 
	\begin{eqnarray}
	v_{t+1} & = & \tau v_{t}+\widetilde{\nabla}_{W}\bar{\mathcal{L}}\left(W_{t},\Gamma_{t}\right)\\
	W_{t+1} & = & W_{t}-\kappa\alpha v_{t+1}
	\end{eqnarray}
	\vspace{-0.1in}
\end{subequations} where $\tau$ is the momentum factor, empirically
setting as 0.9. 


\textbf{DessiLBI  with momentum and weight decay (Mom-Wd). } The update formulation is ($\beta=1e^{-4}$)
\begin{eqnarray}
v_{t+1} & = & \tau v_{t}+\widetilde{\nabla}_{W}\bar{\mathcal{L}}\left(W_{t},\Gamma_{t}\right)\\
W_{t+1} & = & W_{t}-\kappa\alpha v_{t+1}-\beta W_{t}
\vspace{-0.1cm}
\end{eqnarray}

%

\noindent \textbf{Implementation.}  Experiments are conducted over various backbones, \textit{e.g.}, LeNet, AlexNet, VGG, 
and ResNet. For MNIST and Cifar-10,
the default hyper-parameters of DessiLBI  are $\kappa=1$, $\nu=10$ and $\alpha_{k}$
is set as $0.1$, decreased by 1/10 every 30 epochs. In ImageNet-2012,
the DessiLBI  utilizes $\kappa=1$, $\nu=1000$, and $\alpha_{k}$
is initially set as 0.1, decays 1/10 every 30 epochs. We set $\lambda=1$
in Eq. (\ref{Eq:prox-operator}) by default, unless otherwise specified.
On MNIST and Cifar-10,  we have batch size as 128; and for all methods,
the batch size of ImageNet 2012 is 256. The standard data augmentation
implemented in pytorch is applied to Cifar-10 and ImageNet-2012,
as \cite{he2016deep}. The weights of all models are initialized as
\cite{he2015delving}. In the experiments, we define \emph{sparsity} as percentage of non-zero parameters, \textit{i.e.}, the number of non-zero weights dividing the total number
of weights in consideration.  Runnable codes can be downloaded\footnote{https://github.com/corwinliu9669/dS2LBI}.

\subsection{Image Classification \label{subsec:SplitLBI-in-training}}


\textbf{Settings}. We compare different variants of SGD and Adam
in the experiments. By default, the learning rate of competitors is
set as $0.1$ for SGD and its variant and $0.001$ for Adam and its
variants, and gradually decreased by 1/10 every 30 epochs.
(1) Naive SGD: the standard SGD with batch input. (2) SGD with
$\mathit{l}_{1}$ penalty (Lasso). The $\mathit{l}_{1}$ norm is applied to
penalize the weights of SGD by encouraging the sparsity of learned
model, with the regularization parameter of the $\mathit{l}_{1}$ penalty term being set as 
$1e^{-3}$ 
(3) SGD with momentum (Mom): we utilize momentum 0.9 in SGD. (4) SGD
with momentum and weight decay (Mom-Wd): we set the momentum 0.9 and
the standard $\mathit{l}_{2}$ weight decay with the coefficient weight
$1e^{-4}$. (5) SGD with Nesterov (Nesterov): the SGD uses nesterov
momentum 0.9. (6) Naive Adam: it refers to standard  Adam\footnote{In the Appendix of Tab. \ref{table:supervised_imagenet_mnist_cifar}, we further give  more results for Adabound, Adagrad,  Amsgrad, and  Radam, which, we found, are difficulty trained on ImageNet-2012 in practice.}.

The results of image classification are shown in Tab.~\ref{table:supervised_imagenet}.
 Our DessiLBI  variants may achieve comparable or even better performance than SGD variants in 100 epochs, indicating the efficacy in learning dense, over-parameterized models. 

\begin{table}
\begin{centering}
 
\par\end{centering}
\begin{centering}
\begin{tabular}{c}
\begin{tabular}{c|c|cc}
\hline 
\multicolumn{2}{c|}{{\small{}{}Dataset }} & \multicolumn{2}{c}{{\small{}{}ImageNet-2012}}\tabularnewline
\hline 
{\small{}{}Models }  & {\small{}{}Variants }  & {\small{}{}AlexNet }  & {\small{}{}ResNet-18}\tabularnewline
\hline 
\multirow{5}{*}{\emph{\small{}{}SGD}{\small{}{} }} & {\small{}{}Naive}  & {\small{}{}{}--/-- }  & {\small{}{}{}60.76/79.18 }\tabularnewline
 & {\small{}{}$\mathit{l}_{1}$}  & {\small{}{}46.49/65.45} & {\small{}{}51.49/72.45}\tabularnewline
 & {\small{}{}Mom }  & {\small{}{}55.14/78.09 }  & {\small{}{}66.98/86.97 }\tabularnewline
 & {\small{}{}Mom-W}\emph{\small{}{}d$^{\star}$}{\small{}{} }  & {\small{}{}56.55/79.09 }  & {\small{}{}69.76/89.18}\tabularnewline
 & {\small{}{}Nesterov }  & {\small{}{}-/- }  & {\small{}{}70.19/89.30}\tabularnewline
\hline 
\multirow{1}{*}{\emph{\small{}{}Adam}{\small{}{} }} & {\small{}{}Naive}  & {\small{}{}--/-- }  & {\small{}{}59.66/83.28}\tabularnewline
\hline 
\multirow{3}{*}{\emph{\small{}{}}\emph{\small{}DessiLBI}}{\small{}{} } & {\small{}{}Naive}  & {\small{}{}55.06/77.69 }  & {\small{}{}65.26/86.57 }\tabularnewline
 & {\small{}{}Mom }  & {\small{}{}56.23/78.48 }  & {\small{}{}68.55/87.85}\tabularnewline
 & {\small{}{}Mom-Wd }  & \textbf{\small{}{}57.09/79.86 }{\small{} } & \textbf{\small{}{}{}{}70.55/89.56}\tabularnewline
\hline 
\end{tabular}\tabularnewline
\end{tabular}
\par\end{centering}
\begin{centering}
\par\end{centering}
\vspace{-0.1in}
{\small{}\caption{\label{table:supervised_imagenet} Top-1/Top-5 accuracy(\%) on ImageNet-2012. $^{\star}$: results from the
official pytorch website. We use the official pytorch codes to run
the competitors. More results on MNIST/Cifar-10, please refer Tab. \ref{table:supervised_imagenet_mnist_cifar} in supplementary.}
} 
\vspace{-0.1in}
\end{table}


\subsection{Learning Sparse Filters for Interpretation} 

\begin{figure}[htb]
	\centering%
	\begin{tabular}{c}
		\hspace{-0.1in}\includegraphics[width=3in]{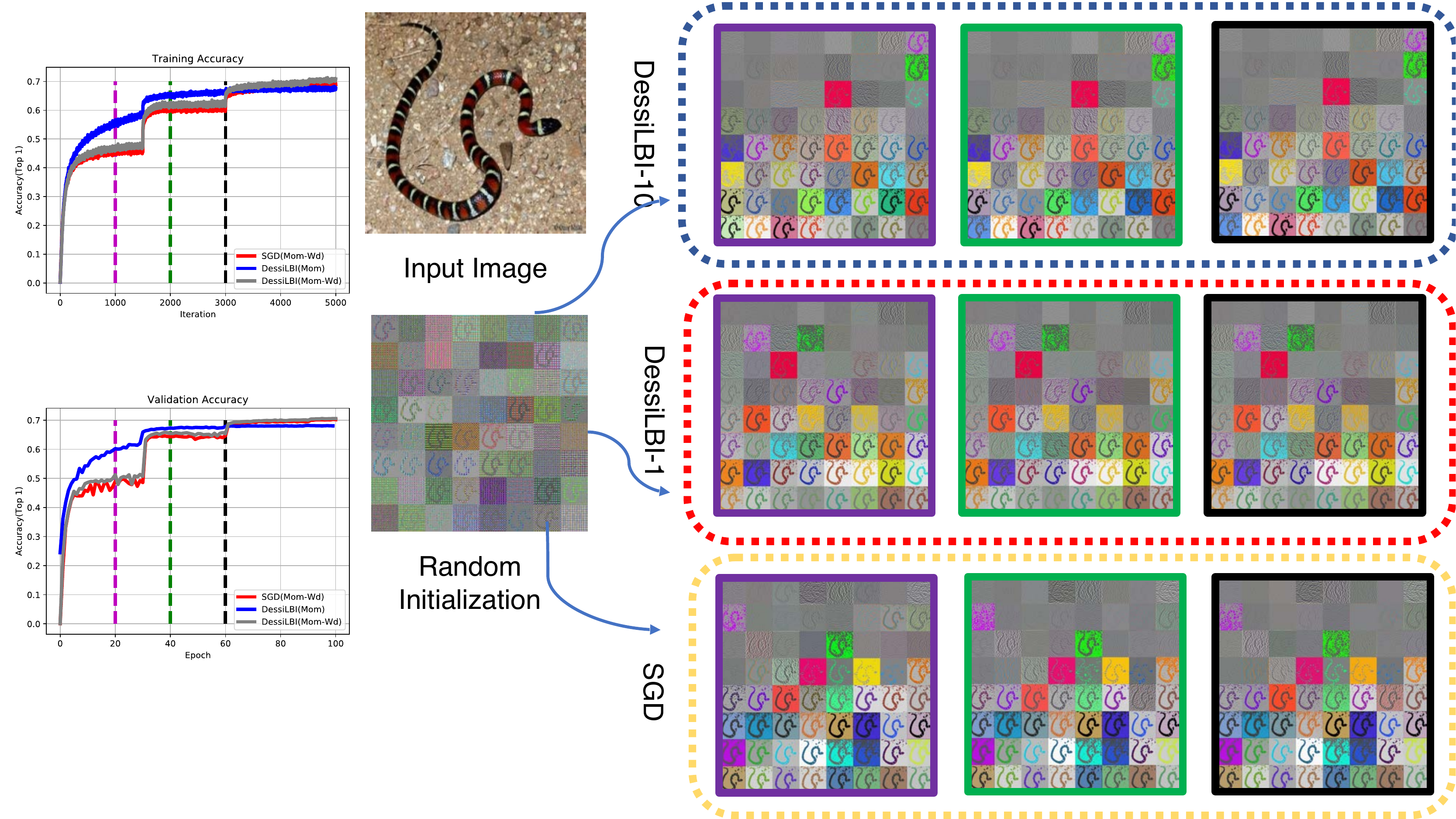}\tabularnewline
	\end{tabular}
	\vspace{-0.15in}
	\caption{\label{fig:imagenet training} Visualization of the first convolutional layer filters of ResNet-18 trained on ImageNet-2012. Given the input
		image and initial weights visualized in the middle, filter response gradients
		at 20 (purple), 40 (green), and 60 (black) epochs are visualized by \cite{springenberg2014striving}. The ``DessiLBI-10'' (``DessiLBI-1'') in the right figure refers to DessiLBI  with $\kappa = 10$ and $\kappa = 1$, respectively. Please refer to Fig. \ref{fig:imagenet-vis} in the Appendix for larger size figure.}
		\vspace{-0.2in}
\end{figure}

\begin{figure*}[t!]
	\centering{}%
	\begin{tabular}{c}
		\includegraphics[width=6.5in]{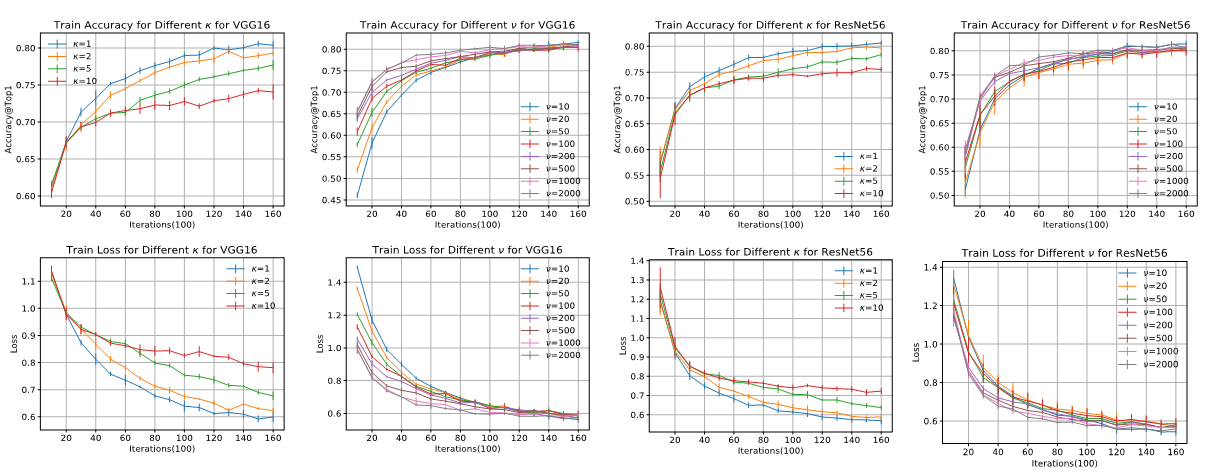} \tabularnewline
	\end{tabular}
		\vspace{-0.15in}
	\caption{Training loss and accuracy curves at different $\kappa$ and $\nu$\label{figure:train}. The X-axis and Y-axis indicate the training epochs, and loss/accuracy. The results are repeated for 5 rounds, by keeping the exactly same initialization for each model. In each round, we use the same initialization for every hyperparameter. For all models, we train for 160 epochs with initial learning rate (lr) of 0.1 and drop by 0.1 at epoch 80 and 120. 
	}
	\vspace{-0.15in}
\end{figure*}

\begin{figure*}[t!]
	\centering{}%
	\begin{tabular}{c}
		\includegraphics[width=6.5in]{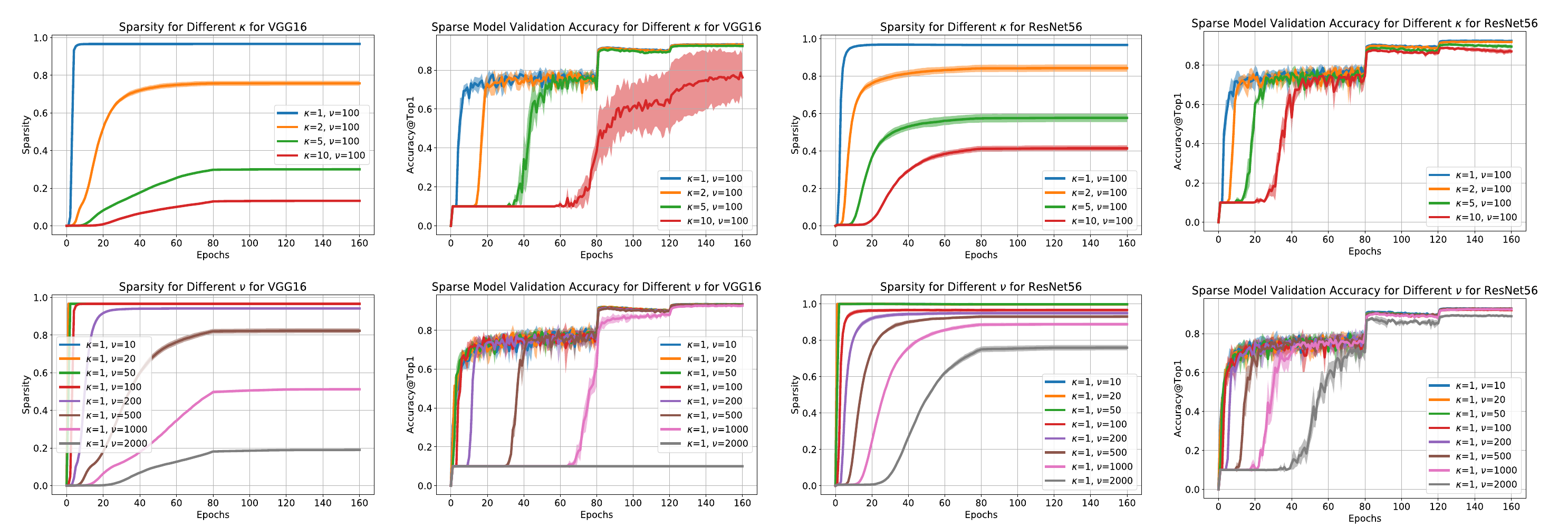} \tabularnewline
	\end{tabular}
		\vspace{-0.14in}
	\caption{Sparsity and validation accuracy by different $\kappa$ and $\nu$\label{figure:spa} show that moderate sparse models may achieve comparable test accuracies to dense models without fine-tuning. Sparsity is obtained as the percentage of nonzeros in $\Gamma_t$ and sparse model at epoch $t$ is obtained by projection of $W_t$ onto the support set of $\Gamma_t$, i.e. pruning the weights corresponding to zeros in $\Gamma_t$. The best accuracies achieved are recorded in comparison with full networks in Tab. \ref{table:ablation_k} and \ref{table:ablation_n} of Appendix for different $\kappa$ and $\nu$, respectively. 
		X-axis and Y-axis indicate the training epochs, and sparsity/accuracy.
		The results are repeated for 5 times. Shaded area indicates the variance;
		and in each round, we keep the exactly same initialization for each
		model. In each round, we use the same initialization for every hyperparameter.  For all the model, we train for 160 epochs with initial learning rate (lr) of 0.1 and decrease by 0.1 at  epoch 80 and 120. }
\vspace{-0.1in}
\end{figure*}

In DessiLBI, the structural sparsity parameter $\Gamma_t$ explores important sub-network architectures that contributes significantly to the loss or error reduction in early training stages. Through the $\ell_2$-coupling, structural sparsity parameter $\Gamma_t$ may guide the weight parameter to explore those sparse models in favour of improved interpretability.  Figure~\ref{mnist_visualization} visualizes some sparse filters learned by DessiLBI  of LeNet-5 trained on MNIST (with $\kappa=10$ and weight decay every $40$ epochs), in comparison with dense filters learned by SGD. The activation pattern of such sparse filters favours high order global correlations between pixels of input images. To further reveal the insights of learned
patterns of DessiLBI, we visualize the first convolutional layer of
ResNet-18 on ImageNet-2012 along the training path of our DessiLBI
as in Fig. \ref{fig:imagenet training}. The left figure compares
the training and validation accuracy of DessiLBI  and SGD. The right
figure compares visualizations of the filters learned by DessiLBI
and SGD. 

\begin{figure*}
	\begin{centering}
		\begin{tabular}{cccc}
			\hspace{-0.15in}\includegraphics[width=0.25\textwidth]{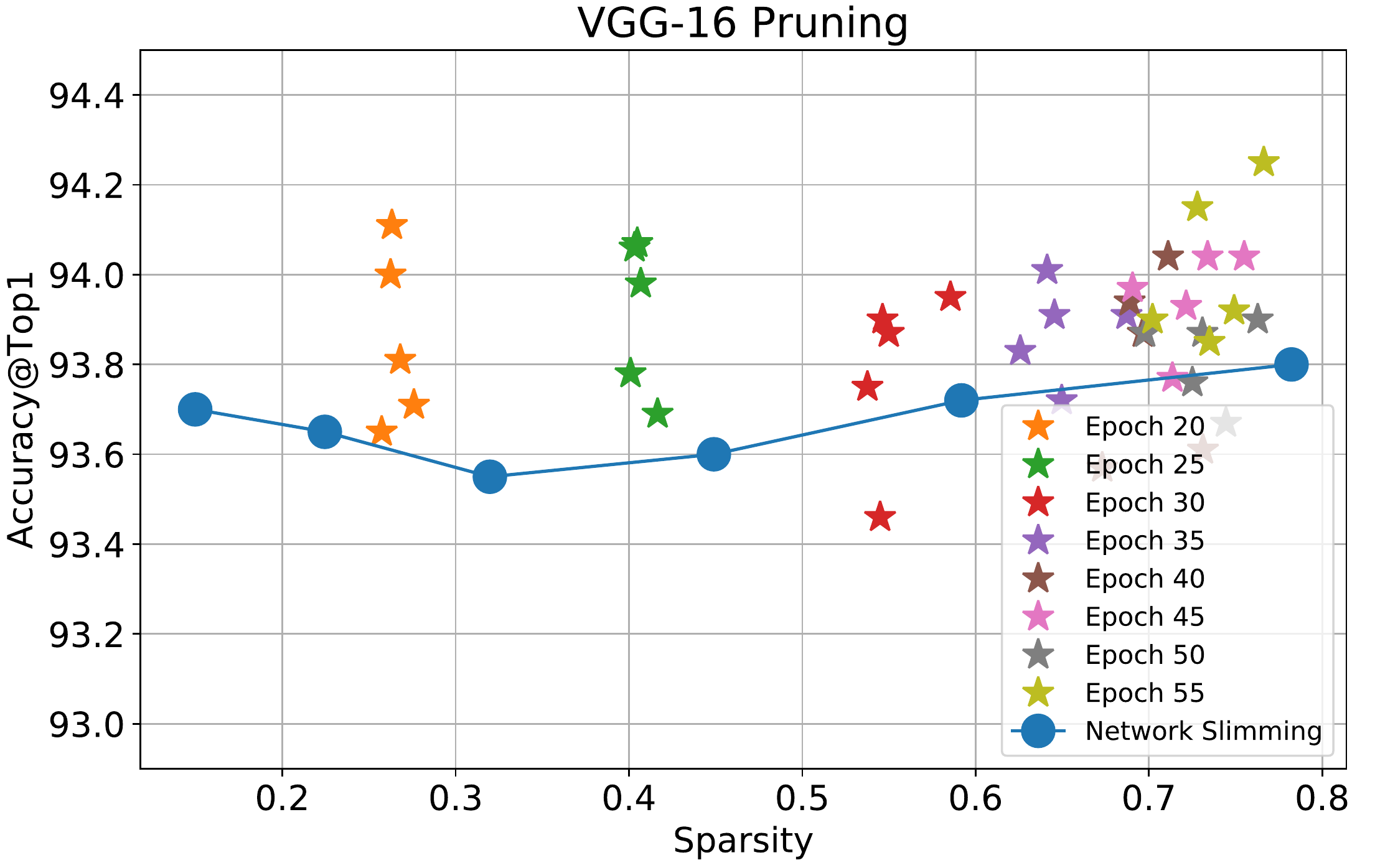}
			\hspace{-0.15in} & \includegraphics[width=0.25\textwidth]{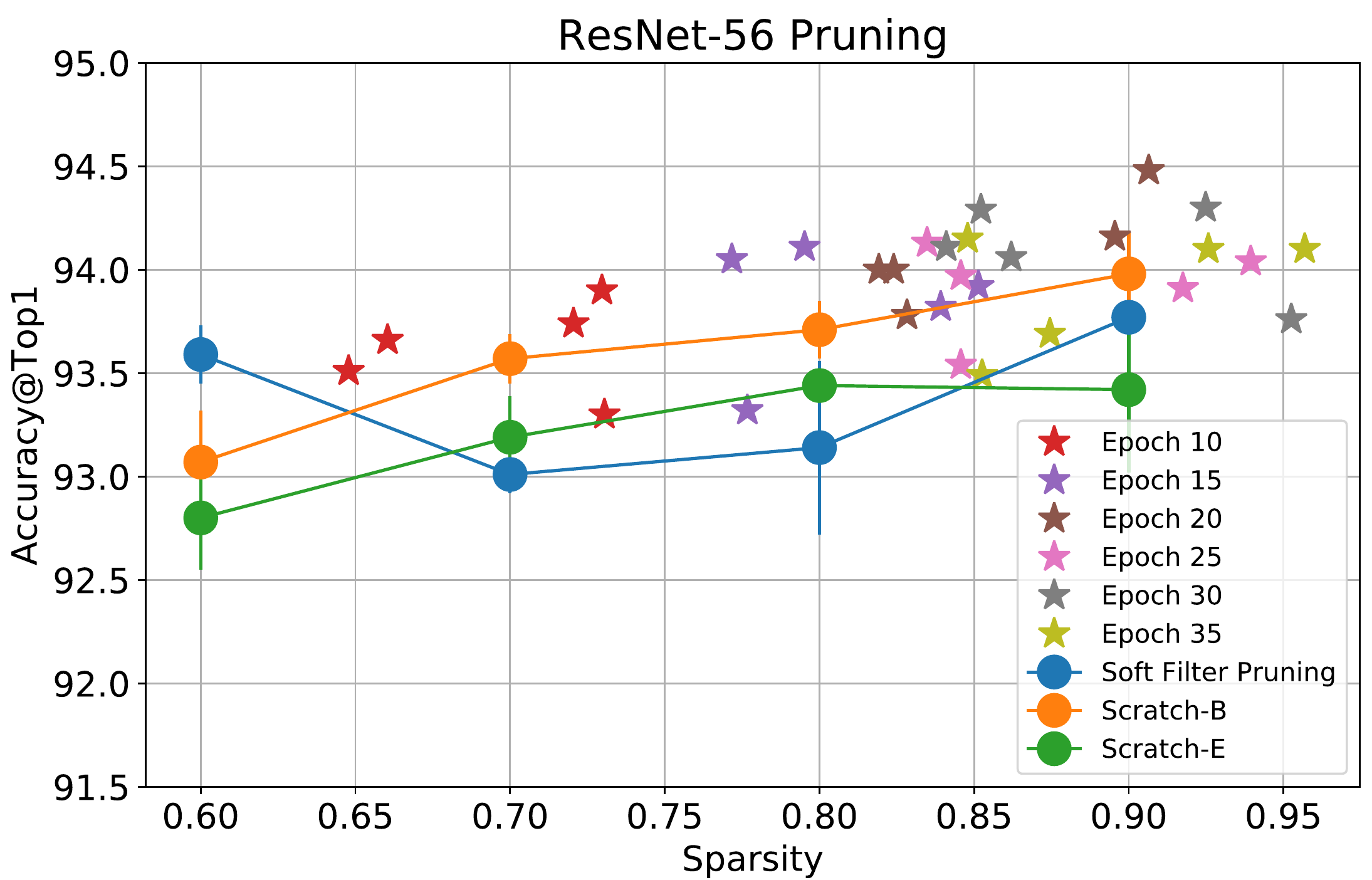}
			\hspace{-0.15in} &  \includegraphics[width=0.25\textwidth]{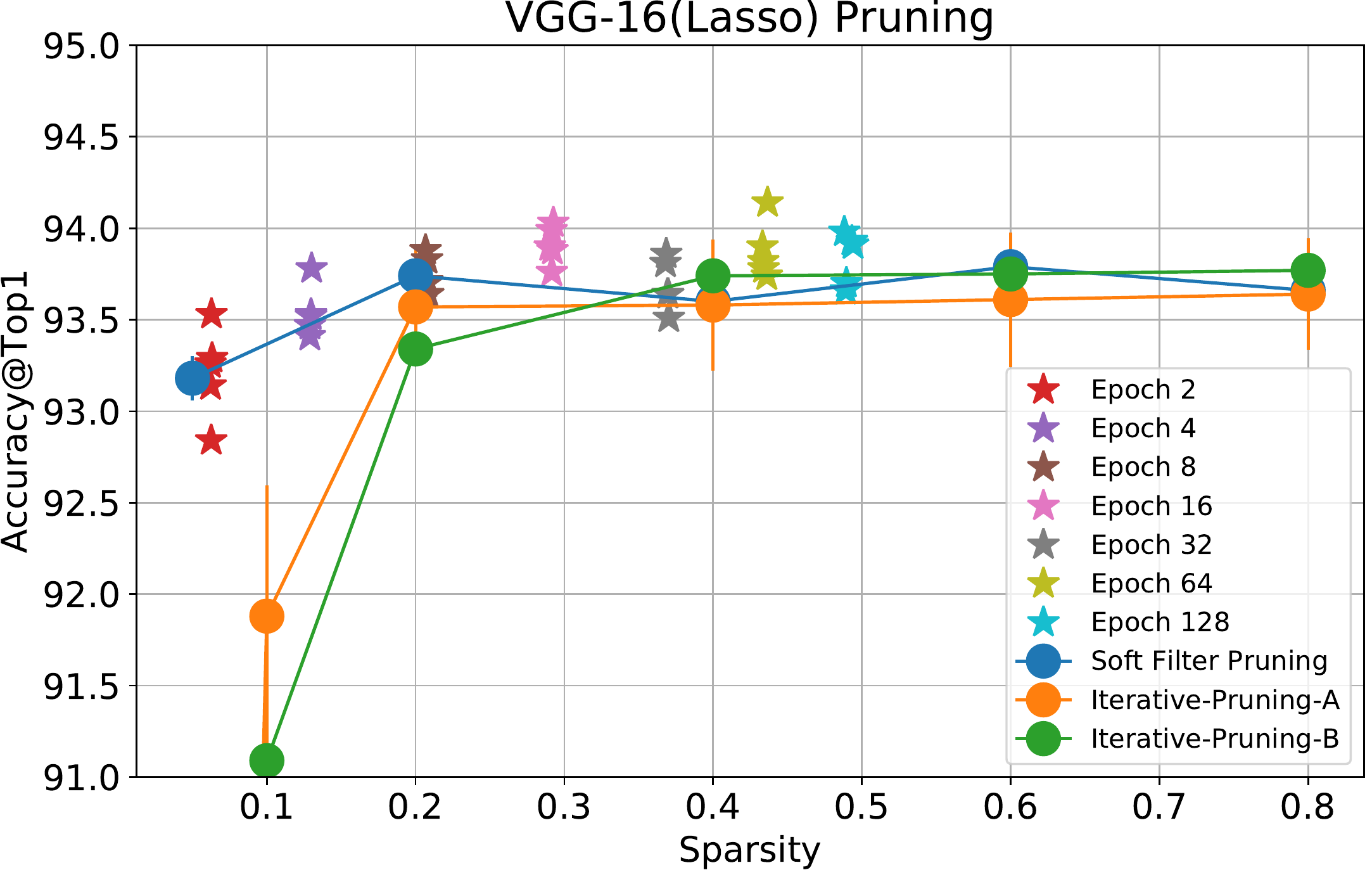}
			\hspace{-0.15in} & \includegraphics[width=0.25\textwidth]{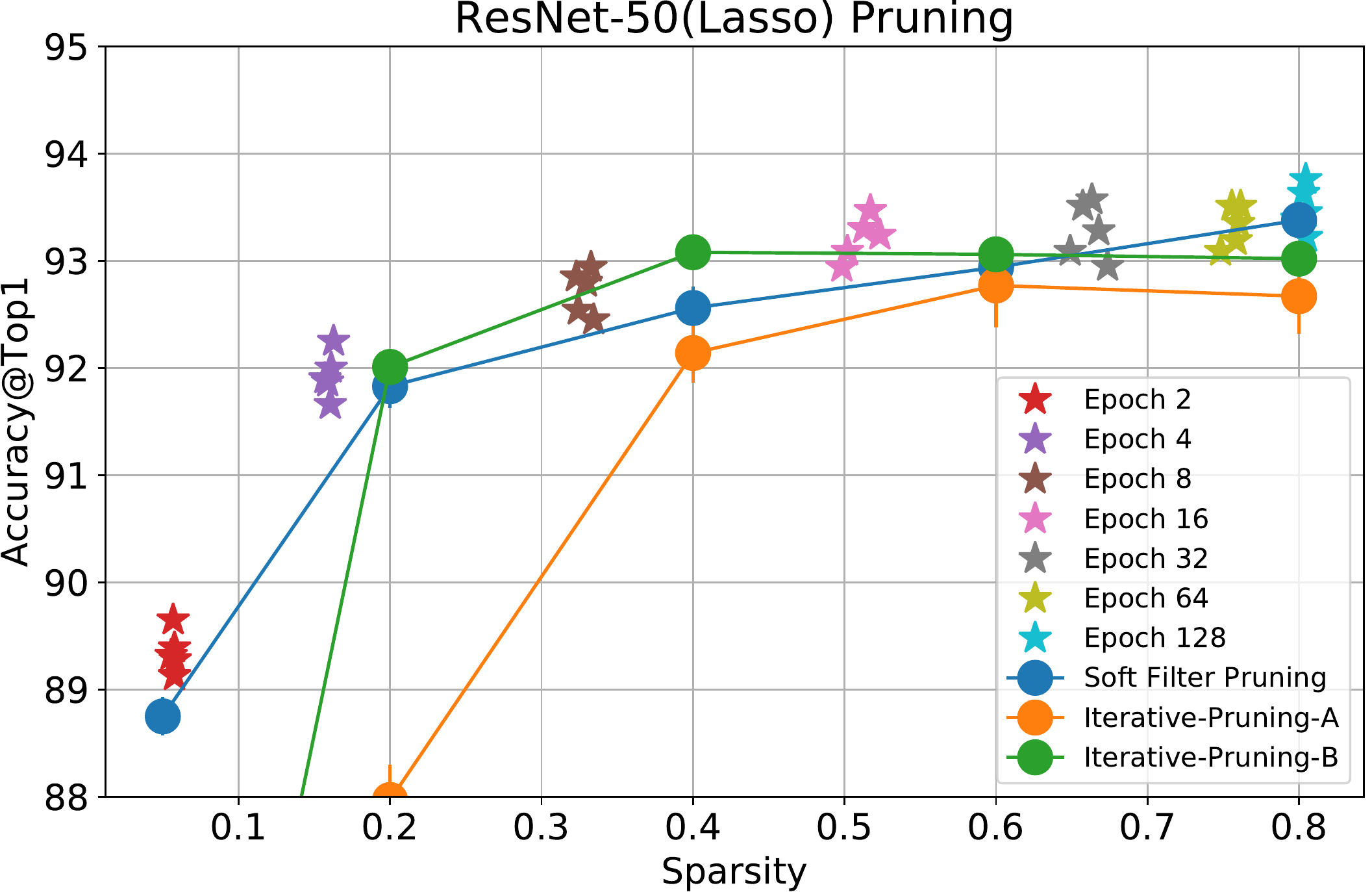}\tabularnewline
			(a) VGG-16  & (b) ResNet-56 & (c) VGG-16 (Lasso) & (d) ResNet-50 (Lasso) \tabularnewline
		\end{tabular}
		\par\end{centering}
	\begin{centering}
		\vspace{-0.15in}
		\caption{DessiLBI  with early stopping finds sparse subnets whose test accuracies (stars) after 
			retrain are comparable or even better than the baselines (Network Slimming (reproduced by the released codes from \cite{rethinking_iclr} )
			, Soft-Filter Pruning (Tab. 10), 
			Scratch-B (Tab. 10), Scratch-E (Tab. 10), and ``Rethinking-Lottery'' (Tab. 9a) as reported in \cite{rethinking_iclr}, Iterative-Pruning-A \cite{han2015learning} and Iterative-Pruning-B \cite{zhu2017prune} (reproduced based on our own implementation)).
			Sparse filters of VGG-16 and ResNet-56 are show in (a) and (b), while sparse weights of VGG-16 and 
			ResNet-50 are shown in (c) and (d). 
			\label{figure:lotteryticketacc} }
		\par\end{centering}
		\vspace{-0.1in}
\end{figure*}

{\bf Visualization.} To be specific, denote
the weights of an $l$-layer network as $\{W^{1},W^{2},\cdots,W^{l}\}$.
For the $i-$th layer weights $W^{i}$, denote the $j-$th channel
$W_{j}^{i}$. Then we compute the gradient of the sum of the feature
map computed from each filter $W_{j}^{i}$ with respect to the input
image (here a snake image). We further conduct the min-max normalization
to the gradient image, and generate the final visualization map. The
right figure compares the visualized gradient images of first convolutional
layer of 64 filters with $7\times7$ receptive fields. We visualize
the models parameters at 20 (purple), 40 (green), and 60 (black) epochs,
respectively, which corresponds to the bounding boxes in the right
figure annotated by the corresponding colors, \emph{i.e.}, purple,
green, and black. We order the gradient images produced from 64 filters
by the descending order of the magnitude ($\ell_{2}$-norm) of filters,
\textit{i.e.}, images are ordered from the upper left to the bottom
right. For comparison, we also provide the visualized gradient from
random initialized weights. 

{\bf DessiLBI learns sparse filters for improved interpreation}.
 Filters learned by ImageNet prefer to non-semantic texture rather than shape and color. The filters of high norms
mostly focus on the texture and shape information, while color information
is with the filters of small magnitudes. This phenomenon is in
accordance with observation of \cite{abbasi2017structural} that filters
mainly of color information can be pruned for saving computational
cost. Moreover, among the filters of high magnitudes, most of them
capture non-semantic textures while few pursue shapes. This shows
that the first convolutional layer of ResNet-18 trained on ImageNet
learned non-semantic textures rather than shape to do image classification
tasks, in accordance with recent studies \citep{TubingenICLR19}. How to enhance
the semantic shape invariance learning, is arguably a key to improve
the robustness of convolutional neural networks.
\vspace{-0.16cm}

\subsection{Training Curves and Structural Sparsity at $(\kappa,\nu)$}

On Cifar-10, we use VGG-16 and ResNet-56 to show the influence of hyperparameters ($\kappa$ and $\nu$) on: (i) training curves (loss and accuracies); and (ii)  structural sparsity learned by $\Gamma_t$.


\textbf{Implementation.} We use DessiLBI  with momentum and weight decay,
due to the good results in Sec.~\ref{subsec:SplitLBI-in-training}.
Specifically, we have these experiments, repeated for 5 times: (1) we fix $\nu=100$
and vary $\kappa=1,2,5,10$, where training curves of $W_t$ are shown in Fig.~\ref{figure:train}, sparsity of $\Gamma_t$ and validation accuracies of sparse models are shown in top row of Fig.~\ref{figure:spa}. Note that we keep $\kappa\cdot\alpha_{k}=0.1$
in Eq (\ref{eq:slbi-iss-show-a}), to make comparable learning rate
of each variant, and also consistent with SGD. Thus  
$\alpha_{k}$ will be adjusted by different $\kappa$. (2) we fix $\kappa=1$,
and change $\nu=10,20,50,100,200,500,1000,2000$
as in Fig.~\ref{figure:train} and the second row of Fig.~\ref{figure:spa} ($\alpha_{k}=0.1$)\footnote{Figure.~\ref{figure:ablation} in Appendix shows  validation accuracies of full models learned by $W_t$.}. 

\textbf{Influence of $\kappa$ and $\nu$ on training curves}. Training loss $(\eL)$ and accuracies in Fig.~\ref{figure:train} converge at different speeds when $\kappa$ and $\nu$ changes. In particular, larger $\kappa$ cause slower convergence, agreeing with the convergence rate in inverse proportion to $\kappa$ suggested in Lemma~\ref{Lemma:relative-error}. Increasing $\nu$ however leads to faster convergence in early epochs, with the advantage vanishing eventually.

\textbf{DessiLBI finds good sparse structure.} Sparse subnetworks achieve comparable performance to dense models without fine-tuning or retraining.  In Fig.~\ref{figure:spa}, the sparsity of $\Gamma$ grows as $\kappa$ and $\nu$ increase. While large $\kappa$ may cause a small number of important parameters growing rapidly, large $\nu$ will decouple $W_t$ and $\Gamma_t$ such that the growth of $W_t$ does not affect $\Gamma_t$ that may over-sparsify and deteriorate model accuracies. Thus a moderate choice of $\kappa$ and $\nu$ is preferred in practice. In Fig.~\ref{figure:spa}, Tab. \ref{table:ablation_k} and \ref{table:ablation_n} in Appendix, one can see that moderate sparse models may achieve comparable predictive power to dense models, even without fine-tuning or retraining. This shows that structural sparsity parameter $\Gamma_t$ can indeed capture important weight parameter $W_t$ through their coupling.


\subsection{Effective Subnetworks by Early Stopping}

With early stopping, $\Gamma_t$ in early epochs may learn effective subnetworks (\emph{i.e.} ``winning tickets" \citep{lottery-iclr19}) that after retraining achieve comparable or even better performance than existing pruning strategies by SGD. 



\textbf{Settings.} On Cifar-10, we adopt one-shot pruning strategy with the backbones of VGG--16,
ResNet-50, and ResNet-56
as \cite{lottery-iclr19}, which firstly trains a dense over-parameterized
model by SGD for $T=160$ epochs and find the sparse structure by pruning weights or filters \citep{rethinking_iclr}, then secondly retrains the
structure from the scratch with $T$ epochs from the same initialization as the first step. For DessiLBI, instead of pruning weights/filters from dense models,
we directly utilize  structural sparsity $\Gamma_t$ at different training epochs to define the subnet architecture, followed by retrain-from-scratch\footnote{Preliminary results  of fine-tuning is  in Appendix Sec. \ref{sec:sparse}. }. 
In particular, we set
$\lambda=0.1,$ and $0.05$ for VGG-16, and ResNet-56 respectively,
since ResNet-56 has less parameters than VGG-16. We further introduce
another variant of our {DessiLBI}  by using Lasso rather than group lasso penalty
for $\Gamma_t$ to sparsify the weights of convolutional filters\footnote{{DessiLBI}  uses momentum and weight decay with hyperparameters shown in Tab. \ref{table:lotterysetting} in Appendix. }, denoting as VGG-16 (Lasso) and ResNet-50 (Lasso), individually. The results are reported over five rounds, as in Fig. \ref{figure:lotteryticketacc}
 Note that in different runs of DessiLBI, the sparsity of $\Gamma_t$ slightly varies.

{\bf Sparse subnets found by early stopping of DessiLBI is effective.} It  achieves remarkably good accuracy after retrain from scratch.
In Fig.\ref{figure:lotteryticketacc} (a-b), sparse filters discovered by $\Gamma_t$ at different epochs are compared against the methods of Network
Slimming \citep{Liu_2017_ICCV}, Soft Filter Pruning \citep{soft_filtering},
Scratch-B, and Scratch-E, whose results are reported from \cite{rethinking_iclr}.
At similar sparsity levels, DessiLBI  can achieve comparable or even better accuracy than competitors, even with sparse architecture learned from very early epochs (e.g. $t=20$ or $10$). Moreover in Fig.\ref{figure:lotteryticketacc} (c-d), we can draw the same
conclusion for the sparse weights of VGG-16 (Lasso) and ResNet-50 (Lasso), against the results reported in \cite{rethinking_iclr}, Iterative-Pruning-A \cite{han2015learning} and Iterative-Pruning-B  \cite{zhu2017prune} (reproduced based on our own implementation)) . These results shows that  structural sparsity  $\Gamma_t$ found by early stopping of {DessiLBI}  already discloses important subnetwork that may achieve remarkably good accuracy after retraining from scratch. Therefore, it is not necessary to fully train a dense model to find a successful sparse subnet architecture with comparable performance to the dense ones, \textit{i.e}., one can early stop DessiLBI  properly where the structural parameter $\Gamma_t$ unveils ``\emph{winning tickets}" \citep{lottery-iclr19}. 

%
%

\section{Conclusion}

This paper presents a novel algorithm -- DessiLBI in exploring structural sparsity of deep network. It is derived from  differential inclusions of inverse scale space, with a proven global convergence to KKT points from arbitrary initializations. \textcolor{black}{Extensive experiments reveal the effectiveness of our algorithm in training over-parameterized models and exploring effective sparse architecture of deep models.}




\section*{Acknowledgement}

The authors would like to thank helpful discussions with Yizhou Wang, Bao Wang, Stanley Osher, Jack Xin, and Wotao Yin. This work was supported in part by NSFC Projects (61977038), Science and Technology Commission of Shanghai Municipality Projects (19511120700, 19ZR1471800), and Shanghai Research and Innovation Functional Program (17DZ2260900). Dr. Zeng is supported by the Two Thousand Talents Plan of Jiangxi Province.
The research of Yuan Yao was supported in part by Hong Kong Research Grant Council (HKRGC) grant 16303817, ITF UIM/390, as well as awards from Tencent AI Lab, Si Family Foundation, and Microsoft Research-Asia.

\newpage
\bibliographystyle{icml2020}
\bibliography{iclr2019_conference,reference,YY_EndNote}

\begin{thebibliography}{82}
\providecommand{\natexlab}[1]{#1}
\providecommand{\url}[1]{\texttt{#1}}
\expandafter\ifx\csname urlstyle\endcsname\relax
  \providecommand{\doi}[1]{doi: #1}\else
  \providecommand{\doi}{doi: \begingroup \urlstyle{rm}\Url}\fi

\bibitem[Abbasi-Asl \& Yu(2017)Abbasi-Asl and Yu]{abbasi2017structural}
Abbasi-Asl, R. and Yu, B.
\newblock Structural compression of convolutional neural networks based on
  greedy filter pruning.
\newblock \emph{arXiv preprint arXiv:1705.07356}, 2017.

\bibitem[{Allen-Zhu} et~al.(2018){Allen-Zhu}, Li, and Song]{Allen-Zhu18}
{Allen-Zhu}, Z., Li, Y., and Song, Z.
\newblock A convergence theory for deep learning via over-parameterization.
\newblock 2018.
\newblock arXiv:1811.03962.

\bibitem[Alvarez \& Salzmann(2016)Alvarez and Salzmann]{group_spars_network}
Alvarez, J.~M. and Salzmann, M.
\newblock Learning the number of neurons in deep networks.
\newblock In \emph{NIPS}, 2016.

\bibitem[Arora et~al.(2018)Arora, Ge, Neyshabur, and Zhang]{arora2018stronger}
Arora, S., Ge, R., Neyshabur, B., and Zhang, Y.
\newblock Stronger generalization bounds for deep nets via a compression
  approach.
\newblock \emph{arXiv preprint arXiv:1802.05296}, 2018.

\bibitem[Attouch et~al.(2013)Attouch, Bolte, and Svaiter]{Attouch2013}
Attouch, H., Bolte, J., and Svaiter, B.~F.
\newblock Convergence of descent methods for semi-algebraic and tame problems:
  proximal algorithms, forward-backward splitting, and regularized
  {G}auss-{S}eidel methods.
\newblock \emph{Mathematical Programming}, 137:\penalty0 91--129, 2013.

\bibitem[Azizan et~al.(2019)Azizan, Lale, and Hassibi]{azizan2019sto}
Azizan, N., Lale, S., and Hassibi, B.
\newblock Stochastic mirror descent on overparameterized nonlinear models:
  Convergence, implicit regularization, and generalization.
\newblock \emph{arXiv preprint arXiv:1906.03830}, 2019.

\bibitem[Bartlett et~al.(2017)Bartlett, Foster, and
  Telgarsky]{bartlett17spectral}
Bartlett, P., Foster, D.~J., and Telgarsky, M.
\newblock Spectrally-normalized margin bounds for neural networks.
\newblock In \emph{The 31st Conference on Neural Information Processing Systems
  (NIPS), Long Beach, CA, USA.} 2017.

\bibitem[Bartlett(1997)]{Bartlett97}
Bartlett, P.~L.
\newblock For valid generalization the size of the weights is more important
  than the size of the network.
\newblock In Mozer, M.~C., Jordan, M.~I., and Petsche, T. (eds.),
  \emph{Advances in Neural Information Processing Systems 9}, pp.\  134--140.
  MIT Press, 1997.

\bibitem[Beck \& Teboulle(2003)Beck and Teboulle]{beck2003mirror}
Beck, A. and Teboulle, M.
\newblock Mirror descent and nonlinear projected subgradient methods for convex
  optimization.
\newblock \emph{Operations Research Letters}, 31\penalty0 (3):\penalty0
  167--175, 2003.

\bibitem[Benning et~al.(2017)Benning, Betcke, Ehrhardt, and
  Sch$\ddot{o}$nlieB]{Benning2017}
Benning, M., Betcke, M.~M., Ehrhardt, M.~J., and Sch$\ddot{o}$nlieB, C.-B.
\newblock Choose your path wisely: gradient descent in a bregman distance
  framework.
\newblock \emph{arXiv preprint arXiv:1712.04045}, 2017.

\bibitem[Bochnak et~al.(1998)Bochnak, Coste, and
  Roy]{Bochnak-semialgebraic1998}
Bochnak, J., Coste, M., and Roy, M.-F.
\newblock \emph{Real algebraic geometry}, volume~3.
\newblock Ergeb. Math. Grenzgeb. Springer-Verlag, Berlin, 1998.

\bibitem[Bolte et~al.(2007{\natexlab{a}})Bolte, Daniilidis, and
  Lewis]{Bolte-KL2007a}
Bolte, J., Daniilidis, A., and Lewis, A.
\newblock {The \L{}ojasiewicz inequality for nonsmooth subanalytic functions
  with applications to subgradient dynamical systems}.
\newblock \emph{SIAM Journal on Optimization}, 17:\penalty0 1205--1223,
  2007{\natexlab{a}}.

\bibitem[Bolte et~al.(2007{\natexlab{b}})Bolte, Daniilidis, Lewis, and
  Shiota]{Bolte-KL2007}
Bolte, J., Daniilidis, A., Lewis, A., and Shiota, M.
\newblock Clark subgradients of stratifiable functions.
\newblock \emph{SIAM Journal on Optimization}, 18:\penalty0 556--572,
  2007{\natexlab{b}}.

\bibitem[Bottou(2010)]{bottou-2010}
Bottou, L.
\newblock Large-scale machine learning with stochastic gradient descent.
\newblock In \emph{COMPSTAT}, 2010.

\bibitem[Boyd et~al.(2011)Boyd, Parikh, Chu, Peleato, Eckstein,
  et~al.]{boyd2011distributed}
Boyd, S., Parikh, N., Chu, E., Peleato, B., Eckstein, J., et~al.
\newblock Distributed optimization and statistical learning via the alternating
  direction method of multipliers.
\newblock \emph{Foundations and Trends{\textregistered} in Machine Learning},
  3\penalty0 (1):\penalty0 1--122, 2011.

\bibitem[Burger et~al.(2006)Burger, Gilboa, Osher, and Xu]{BGOX06}
Burger, M., Gilboa, G., Osher, S., and Xu, J.
\newblock Nonlinear inverse scale space methods.
\newblock \emph{Communications in Mathematical Sciences}, 4\penalty0
  (1):\penalty0 179--212, 2006.

\bibitem[Cai et~al.(2009)Cai, Osher, and Shen]{cai2009convergence}
Cai, J.-F., Osher, S., and Shen, Z.
\newblock Convergence of the linearized bregman iteration for l1-norm
  minimization.
\newblock \emph{Mathematics of Computation}, 2009.

\bibitem[Collins \& Kohli(2014)Collins and Kohli]{l1_memory}
Collins, M. and Kohli, P.
\newblock Memory bounded deep convolutional networks.
\newblock In \emph{arXiv preprint arXiv:1412.1442, 2014}, 2014.

\bibitem[Coste(1999)]{Coste1999-o-minimal}
Coste, M.
\newblock \emph{An introduction to o-minimal geometry}.
\newblock RAAG Notes, 81 pages, Institut de Recherche Mathematiques de Rennes,
  1999.

\bibitem[Donoho \& Huo(2001)Donoho and Huo]{DonHuo01}
Donoho, D.~L. and Huo, X.
\newblock Uncertainty principles and ideal atomic decomposition.
\newblock \emph{IEEE Transactions on Information Theory}, 47\penalty0
  (7):\penalty0 2845--2862, 2001.

\bibitem[Du et~al.(2018)Du, Lee, Li, Wang, and Zhai]{Du18}
Du, S.~S., Lee, J.~D., Li, H., Wang, L., and Zhai, X.
\newblock Gradient descent finds global minima of deep neural networks.
\newblock 2018.
\newblock arXiv:1811.03804.

\bibitem[Erhan et~al.(2009)Erhan, Bengio, Courville, and
  Vincent]{erhan2009visualizing}
Erhan, D., Bengio, Y., Courville, A., and Vincent, P.
\newblock Visualizing higher-layer features of a deep network.
\newblock \emph{University of Montreal, Technical Report}, 1341, 2009.

\bibitem[Franca et~al.(2018)Franca, Robinson, and Vidal]{franca2018admm}
Franca, G., Robinson, D.~P., and Vidal, R.
\newblock Admm and accelerated admm as continuous dynamical systems.
\newblock \emph{arXiv preprint arXiv:1805.06579}, 2018.

\bibitem[Frankle \& Carbin(2019)Frankle and Carbin]{lottery-iclr19}
Frankle, J. and Carbin, M.
\newblock The lottery ticket hypothesis: Finding sparse, trainable neural
  networks.
\newblock \emph{International Conference on Learning Representations (ICLR)},
  2019.
\newblock arXiv preprint arXiv:1803.03635.

\bibitem[Frankle et~al.(2019)Frankle, Dziugaite, Roy, and
  Carbin]{frankle2019lottery}
Frankle, J., Dziugaite, G.~K., Roy, D.~M., and Carbin, M.
\newblock The lottery ticket hypothesis at scale.
\newblock \emph{arXiv preprint arXiv:1903.01611}, 2019.

\bibitem[Geirhos et~al.(2019)Geirhos, Rubisch, Michaelis, Bethge, Wichmann, and
  Brendel]{TubingenICLR19}
Geirhos, R., Rubisch, P., Michaelis, C., Bethge, M., Wichmann, F.~A., and
  Brendel, W.
\newblock Imagenet-trained cnns are biased towards texture; increasing shape
  bias improves accuracy and robustness.
\newblock \emph{International Conference on Learning Representations (ICLR)},
  2019.
\newblock arXiv preprint arXiv:1811.12231.

\bibitem[Ghadimi \& Lan(2012)Ghadimi and Lan]{ghadimi2012optimal}
Ghadimi, S. and Lan, G.
\newblock Optimal stochastic approximation algorithms for strongly convex
  stochastic composite optimization i: A generic algorithmic framework.
\newblock \emph{SIAM Journal on Optimization}, 22\penalty0 (4):\penalty0
  1469--1492, 2012.

\bibitem[Golowich et~al.(2018)Golowich, Rakhlin, and
  Shamir]{sasha18size-independent}
Golowich, N., Rakhlin, A., and Shamir, O.
\newblock Size-independent sample complexity of neural networks.
\newblock \emph{Conference on Learning Theory (COLT)}, 2018.
\newblock arXiv preprint arXiv:1712.06541.

\bibitem[Han et~al.(2015)Han, Pool, Tran, and Dally]{han2015learning}
Han, S., Pool, J., Tran, J., and Dally, W.
\newblock Learning both weights and connections for efficient neural network.
\newblock In \emph{NIPS}, 2015.

\bibitem[He \& Yuan(2012)He and Yuan]{he20121}
He, B. and Yuan, X.
\newblock On the o(1/n) convergence rate of the douglas--rachford alternating
  direction method.
\newblock \emph{SIAM Journal on Numerical Analysis}, 50\penalty0 (2):\penalty0
  700--709, 2012.

\bibitem[He et~al.(2015)He, Zhang, Ren, and Sun]{he2015delving}
He, K., Zhang, X., Ren, S., and Sun, J.
\newblock Delving deep into rectifiers: Surpassing human-level performance on
  imagenet classification.
\newblock In \emph{ICCV}, 2015.

\bibitem[He et~al.(2016)He, Zhang, Ren, and Sun]{he2016deep}
He, K., Zhang, X., Ren, S., and Sun, J.
\newblock Deep residual learning for image recognition.
\newblock In \emph{Proceedings of the IEEE conference on computer vision and
  pattern recognition}, pp.\  770--778, 2016.

\bibitem[Huang \& Yao(2018)Huang and Yao]{huang18_aistats}
Huang, C. and Yao, Y.
\newblock A unified dynamic approach to sparse model selection.
\newblock In \emph{The 21st International Conference on Artificial Intelligence
  and Statistics (AISTATS)}, Lanzarote, Spain, 2018.

\bibitem[Huang et~al.(2016)Huang, Sun, Xiong, and Yao]{huang16_nips}
Huang, C., Sun, X., Xiong, J., and Yao, Y.
\newblock Split lbi: An iterative regularization path with structural sparsity.
\newblock In Lee, D.~D., Sugiyama, M., Luxburg, U.~V., Guyon, I., and Garnett,
  R. (eds.), \emph{Advances in Neural Information Processing Systems (NIPS)
  29}, pp.\  3369--3377. 2016.

\bibitem[Huang et~al.(2018)Huang, Sun, Xiong, and Yao]{huang18_acha}
Huang, C., Sun, X., Xiong, J., and Yao, Y.
\newblock Boosting with structural sparsity: A differential inclusion approach.
\newblock \emph{Applied and Computational Harmonic Analysis}, 2018.
\newblock arXiv preprint arXiv:1704.04833.

\bibitem[Jaderberg et~al.(2014)Jaderberg, Vedaldi, and
  Zisserman]{jaderberg2014speeding}
Jaderberg, M., Vedaldi, A., and Zisserman, A.
\newblock Speeding up convolutional neural networks with low rank expansions.
\newblock In \emph{BMVC}, 2014.

\bibitem[Kingma \& Ba(2015)Kingma and Ba]{kingma2014adam}
Kingma, D. and Ba, J.
\newblock Adam: A method for stochastic optimization.
\newblock In \emph{ICLR}, 2015.

\bibitem[Krantz \& Parks(2002)Krantz and Parks]{Krantz2002-real-analytic}
Krantz, S. and Parks, H.~R.
\newblock \emph{A primer of real analytic functions}.
\newblock Birkh\"{a}user, second edition, 2002.

\bibitem[Krichene et~al.(2015)Krichene, Bayen, and
  Bartlett]{krichene2015accelerated}
Krichene, W., Bayen, A., and Bartlett, P.~L.
\newblock Accelerated mirror descent in continuous and discrete time.
\newblock In \emph{Advances in neural information processing systems}, pp.\
  2845--2853, 2015.

\bibitem[Kurdyka(1998)]{Kurdyka-KL1998}
Kurdyka, K.
\newblock On gradients of functions definable in o-minimal structures.
\newblock \emph{Annales de l'institut Fourier}, 48:\penalty0 769--783, 1998.

\bibitem[Li et~al.(2017)Li, Kadav, Durdanovic, Samet, and Graf]{li2016pruning}
Li, H., Kadav, A., Durdanovic, I., Samet, H., and Graf, H.~P.
\newblock Pruning filters for efficient convnets.
\newblock In \emph{ICLR}, 2017.

\bibitem[Liu et~al.(2017)Liu, Li, Shen, Huang, Yan, and Zhang]{Liu_2017_ICCV}
Liu, Z., Li, J., Shen, Z., Huang, G., Yan, S., and Zhang, C.
\newblock Learning efficient convolutional networks through network slimming.
\newblock In \emph{ICCV}, 2017.

\bibitem[Liu et~al.(2019)Liu, Sun, Zhou, Huang, and Darrell]{rethinking_iclr}
Liu, Z., Sun, M., Zhou, T., Huang, G., and Darrell, T.
\newblock Rethinking the value of network pruning.
\newblock In \emph{ICLR}, 2019.

\bibitem[{\L}ojasiewicz(1963)]{Lojasiewicz-KL1963}
{\L}ojasiewicz, S.
\newblock Une propri\'{e}t\'{e} topologique des sous-ensembles analytiques
  r\'{e}els. {I}n: {L}es \'{E}quations aux d\'{e}riv\'{e}es partielles.
\newblock \emph{\'{E}ditions du centre National de la Recherche Scientifique,
  Paris}, pp.\  87--89, 1963.

\bibitem[{\L}ojasiewicz(1965)]{Lojasiewicz1965-semianalytic}
{\L}ojasiewicz, S.
\newblock \emph{Ensembles semi-analytiques}.
\newblock Institut des Hautes Etudes Scientifiques, 1965.

\bibitem[{\L}ojasiewicz(1993)]{Lojasiewicz-KL1993}
{\L}ojasiewicz, S.
\newblock Sur la geometrie semi-et sous-analytique.
\newblock \emph{Annales de l'institut Fourier}, 43:\penalty0 1575--1595, 1993.

\bibitem[Loshchilov \& Hutter(2019)Loshchilov and Hutter]{LosHut19}
Loshchilov, I. and Hutter, F.
\newblock Decoupled weight decay regularization.
\newblock \emph{International Conference on Learning Representations (ICLR)},
  2019.
\newblock arXiv preprint arXiv:1711.05101.

\bibitem[Mei et~al.(2018)Mei, Montanari, and Nguyen]{Mei_pnas18}
Mei, S., Montanari, A., and Nguyen, P.-M.
\newblock A mean field view of the landscape of two-layers neural network.
\newblock \emph{Proceedings of the National Academy of Sciences (PNAS)}, 2018.

\bibitem[Mei et~al.(2019)Mei, Misiakiewicz, and Montanari]{Mei_colt19}
Mei, S., Misiakiewicz, T., and Montanari, A.
\newblock Mean-field theory of two-layers neural networks: dimension-free
  bounds and kernel limit.
\newblock \emph{Conference on Learning Theory (COLT)}, 2019.

\bibitem[Mordukhovich(2006)]{Mordukhovich-2006}
Mordukhovich, B.~S.
\newblock \emph{Variational analysis and generalized differentiation I: Basic
  Theory}.
\newblock Springer, 2006.

\bibitem[Nedic \& Lee(2014)Nedic and Lee]{nedic2014stochastic}
Nedic, A. and Lee, S.
\newblock On stochastic subgradient mirror-descent algorithm with weighted
  averaging.
\newblock \emph{SIAM Journal on Optimization}, 24\penalty0 (1):\penalty0
  84--107, 2014.

\bibitem[Nemirovski()]{nemirovski2012tutorial}
Nemirovski, A.
\newblock Tutorial: Mirror descent algorithms for large-scale deterministic and
  stochastic convex optimization.

\bibitem[Nemirovski \& Yudin(1983)Nemirovski and Yudin]{NemYu83}
Nemirovski, A. and Yudin, D.
\newblock \emph{Problem complexity and Method Efficiency in Optimization}.
\newblock New York: Wiley, 1983.
\newblock Nauka Publishers, Moscow (in Russian), 1978.

\bibitem[Neyshabur et~al.(2019)Neyshabur, Li, Bhojanapalli, LeCun, and
  Srebro]{srebro19_iclr}
Neyshabur, B., Li, Z., Bhojanapalli, S., LeCun, Y., and Srebro, N.
\newblock The role of over-parametrization in generalization of neural
  networks.
\newblock In \emph{International Conference on Learning Representations (ICLR),
  New Orleans, Louisiana, USA.} 2019.

\bibitem[Osher et~al.(2005)Osher, Burger, Goldfarb, Xu, and
  Yin]{osher2005iterative}
Osher, S., Burger, M., Goldfarb, D., Xu, J., and Yin, W.
\newblock An iterative regularization method for total variation-based image
  restoration.
\newblock \emph{Multiscale Modeling \& Simulation}, 4\penalty0 (2):\penalty0
  460--489, 2005.

\bibitem[Osher et~al.(2016)Osher, Ruan, Xiong, Yao, and Yin]{osher2016diff}
Osher, S., Ruan, F., Xiong, J., Yao, Y., and Yin, W.
\newblock Sparse recovery via differential inclusions.
\newblock \emph{Applied and Computational Harmonic Analysis}, 2016.

\bibitem[Rockafellar \& Wets(1998)Rockafellar and Wets]{Rockafellar1998}
Rockafellar, R.~T. and Wets, R. J.-B.
\newblock \emph{Variational analysis}.
\newblock Grundlehren Math. Wiss. 317, Springer-Verlag, New York, 1998.

\bibitem[Shiota(1997)]{Shiota1997-subanalytic}
Shiota, M.
\newblock \emph{Geometry of subanalytic and semialgebraic sets}, volume 150 of
  \emph{Progress in Mathematics}.
\newblock Birkh\"{a}user, Boston, 1997.

\bibitem[Springenberg et~al.(2014)Springenberg, Dosovitskiy, Brox, and
  Riedmiller]{springenberg2014striving}
Springenberg, J.~T., Dosovitskiy, A., Brox, T., and Riedmiller, M.
\newblock Striving for simplicity: The all convolutional net.
\newblock \emph{arXiv preprint arXiv:1412.6806}, 2014.

\bibitem[Su et~al.(2016)Su, Boyd, and Candes]{su2016differential}
Su, W., Boyd, S., and Candes, E.~J.
\newblock A differential equation for modeling nesterov's accelerated gradient
  method: theory and insights.
\newblock \emph{The Journal of Machine Learning Research}, 17\penalty0
  (1):\penalty0 5312--5354, 2016.

\bibitem[Tropp(2004)]{Tropp04}
Tropp, J.~A.
\newblock Greed is good: Algorithmic results for sparse approximation.
\newblock \emph{IEEE Trans. Inform. Theory}, 50\penalty0 (10):\penalty0
  2231--2242, 2004.

\bibitem[van~den Dries(1986)]{vandenDries1986-o-minimal}
van~den Dries, L.
\newblock A generalization of the tarski-seidenberg theorem and some
  nondefinability results.
\newblock \emph{Bull. Amer. Math. Soc. (N.S.)}, 15:\penalty0 189--193, 1986.

\bibitem[van~den Dries \& Miller(1996)van~den Dries and
  Miller]{vandenDries1996-GC}
van~den Dries, L. and Miller, C.
\newblock Geometric categories and o-minimal structures.
\newblock \emph{Duke Mathematical Journal}, 84:\penalty0 497--540, 1996.

\bibitem[Venturi et~al.(2018)Venturi, Bandeira, and Bruna]{VenBanBru18}
Venturi, L., Bandeira, A.~S., and Bruna, J.
\newblock Spurious valleys in two-layer neural network optimization landscapes.
\newblock 2018.
\newblock arXiv:1802.06384.

\bibitem[Wahlberg et~al.(2012)Wahlberg, Boyd, Annergren, and
  Wang]{wahlberg2012admm}
Wahlberg, B., Boyd, S., Annergren, M., and Wang, Y.
\newblock An admm algorithm for a class of total variation regularized
  estimation problems.
\newblock \emph{IFAC Proceedings Volumes}, 45\penalty0 (16):\penalty0 83--88,
  2012.

\bibitem[Wang \& Banerjee(2013)Wang and Banerjee]{wang2013online}
Wang, H. and Banerjee, A.
\newblock Online alternating direction method (longer version).
\newblock \emph{arXiv preprint arXiv:1306.3721}, 2013.

\bibitem[Wang \& Banerjee(2014)Wang and Banerjee]{wang2014bregman}
Wang, H. and Banerjee, A.
\newblock Bregman alternating direction method of multipliers.
\newblock In \emph{Advances in Neural Information Processing Systems}, pp.\
  2816--2824, 2014.

\bibitem[Wang et~al.(2019)Wang, Yin, and Zeng]{Wang-ADMM2018}
Wang, Y., Yin, W., and Zeng, J.
\newblock Global convergence of admm in nonconvex nonsmooth optimization.
\newblock \emph{Journal of Scientific Computing}, 78\penalty0 (1):\penalty0
  29--63, 2019.

\bibitem[Wei et~al.(2017)Wei, Yang, and Wainwright]{yuting17-nips}
Wei, Y., Yang, F., and Wainwright, M.~J.
\newblock Early stopping for kernel boosting algorithms: A general analysis
  with localized complexities.
\newblock \emph{The 31st Conference on Neural Information Processing Systems
  (NIPS), Long Beach, CA, USA}, 2017.

\bibitem[Wen et~al.(2016)Wen, Wu, Wang, Chen, and Li]{l12_norm}
Wen, W., Wu, C., Wang, Y., Chen, Y., and Li, H.
\newblock Learning the number of neurons in deep networks.
\newblock In \emph{NIPS}, 2016.

\bibitem[Xue \& Xin(2018)Xue and Xin]{xin2018rvsm}
Xue, F. and Xin, J.
\newblock Convergence of a relaxed variable splitting method for learning
  sparse neural networks via $\ell_1$, $\ell_0$, and transformed-$\ell_1$
  penalties.
\newblock \emph{arXiv:1812.05719v2}, 2018.
\newblock URL \url{http://arxiv.org/abs/1812.05719}.

\bibitem[Yang et~al.(2018)Yang, Kang, Dong, Fu, and Yang]{soft_filtering}
Yang, H., Kang, G., Dong, X., Fu, Y., and Yang, Y.
\newblock Soft filter pruning for accelerating deep convolutional neural
  networks.
\newblock In \emph{IJCAI 2018}, 2018.

\bibitem[Yao et~al.(2007)Yao, Rosasco, and Caponnetto]{YaoRosCap07}
Yao, Y., Rosasco, L., and Caponnetto, A.
\newblock On early stopping in gradient descent learning.
\newblock \emph{Constructive Approximation}, 26\penalty0 (2):\penalty0
  289--315, 2007.

\bibitem[Yin et~al.(2008)Yin, Osher, Darbon, and Goldfarb]{yin2008bregman}
Yin, W., Osher, S., Darbon, J., and Goldfarb, D.
\newblock Bregman iterative algorithms for compressed sensing and related
  problems.
\newblock \emph{SIAM Journal on Imaging sciences}, 1\penalty0 (1):\penalty0
  143--168, 2008.

\bibitem[Yoon \& Hwang(2017)Yoon and Hwang]{yoon2017combined}
Yoon, J. and Hwang, S.~J.
\newblock Combined group and exclusive sparsity for deep neural networks.
\newblock In \emph{ICML}, 2017.

\bibitem[Yuan \& Lin(2006)Yuan and Lin]{yuan2006model}
Yuan, M. and Lin, Y.
\newblock Model selection and estimation in regression with grouped variables.
\newblock \emph{Journal of the Royal Statistical Society: Series B (Statistical
  Methodology)}, 68\penalty0 (1):\penalty0 49--67, 2006.

\bibitem[Zeng et~al.(2019{\natexlab{a}})Zeng, Lau, Lin, and Yao]{Zeng2019}
Zeng, J., Lau, T. T.-K., Lin, S.-B., and Yao, Y.
\newblock Global convergence of block coordinate descent in deep learning.
\newblock In \emph{Proceedings of the 36th International Conference on Machine
  Learning, Long Beach, California}, 2019{\natexlab{a}}.
\newblock URL \url{https://arxiv.org/abs/1803.00225}.

\bibitem[Zeng et~al.(2019{\natexlab{b}})Zeng, Lin, and
  Yao]{zeng2019convergence}
Zeng, J., Lin, S.-B., and Yao, Y.
\newblock A convergence analysis of nonlinearly constrained admm in deep
  learning.
\newblock \emph{arXiv preprint arXiv:1902.02060}, 2019{\natexlab{b}}.

\bibitem[Zhang et~al.(2017)Zhang, Bengio, Hardt, Recht, and
  Vinyals]{zhang16rethinking}
Zhang, C., Bengio, S., Hardt, M., Recht, B., and Vinyals, O.
\newblock Understanding deep learning requires rethinking generalization.
\newblock \emph{International Conference on Learning Representations (ICLR)},
  2017.
\newblock arXiv:1611.03530.

\bibitem[Zhao \& Yu(2006)Zhao and Yu]{ZhaYu06}
Zhao, P. and Yu, B.
\newblock On model selection consistency of lasso.
\newblock \emph{J. Machine Learning Research}, 7:\penalty0 2541--2567, 2006.

\bibitem[Zhu \& Gupta(2017)Zhu and Gupta]{zhu2017prune}
Zhu, M. and Gupta, S.
\newblock To prune, or not to prune: exploring the efficacy of pruning for
  model compression, 2017.

\bibitem[Zhu et~al.(2018)Zhu, Huang, and Yao]{ZHY18}
Zhu, W., Huang, Y., and Yao, Y.
\newblock On breiman's dilemma in neural networks: Phase transitions of margin
  dynamics.
\newblock \emph{arXiv:1810.03389}, 2018.

\end{thebibliography}

\newpage\onecolumn

\appendix


\section*{Appendix to \emph{DessiLBI for deep learning: structural sparsity via differential inclusion paths}}


\section{Proof of Theorem \ref{Thm:conv-SLBI}}

\label{sc:proof}

First of all, we reformulate Eq.~(\ref{Eq:SLBI-reformulation}) into
an equivalent form. Without loss of generality, consider $\Omega=\Omega_1$ in the sequel. 

Denote $R(P):=\Omega(\Gamma)$, then Eq. (\ref{Eq:SLBI-reformulation})
can be rewritten as, DessiLBI
\begin{subequations} 
\begin{align}
 & P_{k+1}=\mathrm{Prox}_{\kappa R}(P_{k}+\kappa(p_{k}-\alpha\nabla\bar{\calL}(P_{k}))), \label{Eq:SLBI-reform2-iter1}\\
 & p_{k+1}=p_{k}-\kappa^{-1}(P_{k+1}-P_{k}+\kappa\alpha\nabla\bar{\mathcal{L}}(P_{k})),\label{Eq:SLBI-reform2-iter2}
\end{align}
\end{subequations} where $p_{k}=[0,g_{k}]^{T}\in\partial R(P_{k})$
and $g_{k}\in\partial\Omega(\Gamma_{k})$. Thus DessiLBI is equivalent
to the following iterations, 
\begin{subequations} 
\begin{align}
 & W_{k+1}=W_{k}-\kappa\alpha\nabla_{W}\bar{\mathcal{L}}(W_{k},\Gamma_{k}),\label{Eq:SLBI-reform-iter1}\\
 & \Gamma_{k+1}=\mathrm{Prox}_{\kappa\Omega}(\Gamma_{k}+\kappa(g_{k}-\alpha\nabla_{\Gamma}\bar{\mathcal{L}}(W_{k},\Gamma_{k}))),\label{Eq:SLBI-reform-iter2}\\
 & g_{k+1}=g_{k}-\kappa^{-1}(\Gamma_{k+1}-\Gamma_{k}+\kappa\alpha\cdot\nabla_{\Gamma}\bar{\mathcal{L}}(W_{k},\Gamma_{k})).\label{Eq:SLBI-reform-iter3}
\end{align}
\end{subequations} 

Exploiting the equivalent reformulation (\ref{Eq:SLBI-reform-iter1}-\ref{Eq:SLBI-reform-iter3}),
one can establish the global convergence of $(W_{k},\Gamma_{k},g_{k})$ based on the Kurdyka-{\L }ojasiewicz framework. In this section, the following extended version
of Theorem \ref{Thm:conv-SLBI} is actually proved. 
 
\begin{thm}{[}Global Convergence of DessiLBI{]} \label{Thm:conv-SLBI+} Suppose that Assumption
\ref{Assumption} holds. Let $(W_{k},\Gamma_{k},g_{k})$ be the sequence
generated by DessiLBI (Eq. (\ref{Eq:SLBI-reform-iter1}-\ref{Eq:SLBI-reform-iter3}))
with a finite initialization. If 
\begin{align*}
0<\alpha_{k}=\alpha<\frac{2}{\kappa(Lip+\nu^{-1})},
\end{align*}
then $(W_{k},\Gamma_{k},g_{k})$ converges to a critical point of
$F$. Moreover, $\{(W_{k},\Gamma_{k})\}$ converges to a stationary
point of $\bar{\mathcal{L}}$ defined in Eq. \ref{eq:sparse_loss},
and $\{W^{k}\}$ converges to a stationary point of $\eL(W)$. 
\end{thm}

\subsection{Kurdyka-{\L }ojasiewicz Property \label{subsec:Kurdyka-property}}


To introduce the definition of the Kurdyka-{\L }ojasiewicz (KL)
property, we need some notions and notations from variational analysis,
which can be found in \cite{Rockafellar1998}.

The notion of subdifferential plays a central role in the following
definitions. For each ${\bf x}\in\mathrm{dom}(h):=\{{\bf x}\in\mathbb{R}^{p}:h({\bf x})<+\infty\}$,
the \textit{Fr\'{e}chet subdifferential} of $h$ at ${\bf x}$, written
$\widehat{\partial}h({\bf x)}$, is the set of vectors ${\bf v}\in\mathbb{R}^{p}$
which satisfy
\[
\lim\inf_{{\bf y}\neq{\bf x},{\bf y}\rightarrow{\bf x}}\ \frac{h({\bf y})-h({\bf x})-\langle{\bf v},{\bf y}-{\bf x}\rangle}{\|{\bf x}-{\bf y}\|}\geq0.
\]
When ${\bf x}\notin\mathrm{dom}(h),$ we set $\widehat{\partial}h({\bf x})=\varnothing.$
The \emph{limiting-subdifferential} (or simply \emph{subdifferential})
of $h$ introduced in \cite{Mordukhovich-2006}, written $\partial h({\bf x})$
at ${\bf x}\in\mathrm{dom}(h)$, is defined by
\begin{align}
\partial h({\bf x}):=\{{\bf v}\in\mathbb{R}^{p}:\exists{\bf x}^{k}\to{\bf x},\;h({\bf x}^{k})\to h({\bf x}),\;{\bf v}^{k}\in\widehat{\partial}h({\bf x}^{k})\to{\bf v}\}.\label{Def:limiting-subdifferential}
\end{align}
A necessary (but not sufficient) condition for ${\bf x}\in\mathbb{R}^{p}$
to be a minimizer of $h$ is $\mathbf{0}\in\partial h({\bf x})$.
A point that satisfies this inclusion is called \textit{limiting-critical}
or simply \textit{critical}. The distance between a point ${\bf x}$
to a subset ${\cal S}$ of $\mathbb{R}^{p}$, written $\mathrm{dist}({\bf x},{\cal S})$,
is defined by $\mathrm{dist}({\bf x},{\cal S})=\inf\{\|{\bf x}-{\bf s}\|:{\bf s}\in{\cal S}\}$,
where $\|\cdot\|$ represents the Euclidean norm.

Let $h:\mathbb{R}^{p}\to\mathbb{R}\cup\{+\infty\}$ be an extended-real-valued
function (respectively, $h:\mathbb{R}^{p}\rightrightarrows\mathbb{R}^{q}$
be a point-to-set mapping), its \textit{graph} is defined by
\begin{align*}
 & \mathrm{Graph}(h):=\{({\bf x},y)\in\mathbb{R}^{p}\times\mathbb{R}:y=h({\bf x})\},\\
(\text{resp.}\; & \mathrm{Graph}(h):=\{({\bf x},{\bf y})\in\mathbb{R}^{p}\times\mathbb{R}^{q}:{\bf y}\in h({\bf x})\}),
\end{align*}
and its domain by $\mathrm{dom}(h):=\{{\bf x}\in\mathbb{R}^{p}:h({\bf x})<+\infty\}$
(resp. $\mathrm{dom}(h):=\{{\bf x}\in\mathbb{R}^{p}:h({\bf x})\neq\varnothing\}$).
When $h$ is a proper function, i.e., when $\mathrm{dom}(h)\neq\varnothing,$
the set of its global minimizers (possibly empty) is denoted by
\[
\arg\min h:=\{{\bf x}\in\mathbb{R}^{p}:h({\bf x})=\inf h\}.
\]

The KL property \citep{Lojasiewicz-KL1963,Lojasiewicz-KL1993,Kurdyka-KL1998,Bolte-KL2007a,Bolte-KL2007}
plays a central role in the convergence analysis of nonconvex algorithms
\citep{Attouch2013,Wang-ADMM2018}. The following definition is adopted
from \cite{Bolte-KL2007}.

\begin{definition}{[}Kurdyka-{\L }ojasiewicz property{]} \label{def:KL-function}
A function $h$ is said to have the Kurdyka-{\L }ojasiewicz (KL)
property at $\bar{u}\in\mathrm{dom}(\partial h):=\{v\in\mathbb{R}^{n}|\partial h(v)\neq\emptyset\}$,
if there exists a constant $\eta\in(0,\infty)$, a neighborhood ${\cal N}$
of $\bar{u}$ and a function $\phi:[0,\eta)\rightarrow\mathbb{R}_{+}$,
which is a concave function that is continuous at $0$ and satisfies
$\phi(0)=0$, $\phi\in{\cal C}^{1}((0,\eta))$, i.e., $\phi$ is continuous
differentiable on $(0,\eta)$, and $\phi'(s)>0$ for all $s\in(0,\eta)$,
such that for all $u\in{\cal N}\cap\{u\in\mathbb{R}^{n}|h(\bar{u})<h(u)<h(\bar{u})+\eta\}$,
the following inequality holds
\begin{align}
\phi'(h(u)-h(\bar{u}))\cdot\mathrm{dist}(0,\partial h(u))\geq1.\label{Eq:def-KL-function}
\end{align}
If $h$ satisfies the KL property at each point of $\mathrm{dom}(\partial h)$,
$h$ is called a KL function. \end{definition}

KL functions include real analytic functions, semialgebraic functions,
tame functions defined in some o-minimal structures \citep{Kurdyka-KL1998,Bolte-KL2007},
continuous subanalytic functions \citep{Bolte-KL2007a} and locally
strongly convex functions. In the following, we provide some important
examples that satisfy the Kurdyka-{\L }ojasiewicz property.

\begin{definition}{[}Real analytic{]} \label{Def:real-analytic}
A function $h$ with domain an open set $U\subset\mathbb{R}$ and
range the set of either all real or complex numbers, is said to be
\textbf{real analytic} at $u$ if the function $h$ may be represented
by a convergent power series on some interval of positive radius centered
at $u$: $h(x)=\sum_{j=0}^{\infty}\alpha_{j}(x-u)^{j},$ for some
$\{\alpha_{j}\}\subset\RR$. The function is said to be \textbf{real
analytic} on $V\subset U$ if it is real analytic at each $u\in V$
\citep[Definition 1.1.5]{Krantz2002-real-analytic}. The real analytic
function $f$ over $\mathbb{R}^{p}$ for some positive integer $p>1$
can be defined similarly.

According to \cite{Krantz2002-real-analytic}, typical real analytic
functions include polynomials, exponential functions, and the logarithm,
trigonometric and power functions on any open set of their domains.
One can verify whether a multivariable real function $h({\bf x)}$
on $\mathbb{R}^{p}$ is analytic by checking the analyticity of $g(t):=h({\bf x}+t{\bf y})$
for any ${\bf x},{\bf y}\in\mathbb{R}^{p}$. \end{definition}

\begin{definition}{[}Semialgebraic{]}\hfill{}\label{Def:semialgebraic}
\begin{enumerate}
\item[(a)] A set ${\cal D}\subset\mathbb{R}^{p}$ is called semialgebraic \citep{Bochnak-semialgebraic1998}
if it can be represented as
\[
{\cal D}=\bigcup_{i=1}^{s}\bigcap_{j=1}^{t}\left\lbrace {\bf x}\in\mathbb{R}^{p}:P_{ij}({\bf x})=0,Q_{ij}({\bf x})>0\right\rbrace ,
\]
where $P_{ij},Q_{ij}$ are real polynomial functions for $1\leq i\leq s,1\leq j\leq t.$
\item[(b)] A function $h:\mathbb{R}^{p}\rightarrow\mathbb{R}\cup\{+\infty\}$
(resp. a point-to-set mapping $h:\mathbb{R}^{p}\rightrightarrows\mathbb{R}^{q}$)
is called \textit{semialgebraic} if its graph $\mathrm{Graph}(h)$
is semialgebraic.
\end{enumerate}
\end{definition}

According to \citep{Lojasiewicz1965-semianalytic,Bochnak-semialgebraic1998}
and \citep[I.2.9, page 52]{Shiota1997-subanalytic}, the class of semialgebraic
sets are stable under the operation of finite union, finite intersection,
Cartesian product or complementation. Some typical examples include
\text{polynomial} functions, the indicator function of a semialgebraic
set, and the \text{Euclidean norm} \citep[page 26]{Bochnak-semialgebraic1998}.

\subsection{KL Property in Deep Learning and Proof of Corollary \ref{Corollary:DL}}

\label{sc:convergence-DL}

In the following, we consider the deep neural network training problem.
Consider a $l$-layer feedforward neural network including $l-1$
hidden layers of the neural network. Particularly, let $d_{i}$ be
the number of hidden units in the $i$-th hidden layer for $i=1,\ldots,l-1$.
Let $d_{0}$ and $d_{l}$ be the number of units of input and output
layers, respectively. Let $W^{i}\in\RR^{d_{i}\times d_{i-1}}$ be
the weight matrix between the $(i-1)$-th layer and the $i$-th layer
for any $i=1,\ldots l$\footnote{To simplify notations, we regard the input and output layers as the
$0$-th and the $l$-th layers, respectively, and absorb the bias
of each layer into $W^{i}$.}. 

According to Theorem \ref{Thm:conv-SLBI+}, one major condition is
to verify the introduced Lyapunov function $F$ defined in (\ref{Eq:Lyapunov-fun})
satisfies the Kurdyka-{\L }ojasiewicz property. For this purpose,
we need an extension of semialgebraic set, called the \textit{o-minimal
structure} (see, for instance \cite{Coste1999-o-minimal}, \cite{vandenDries1986-o-minimal},
\cite{Kurdyka-KL1998}, \cite{Bolte-KL2007}). The following definition
is from \cite{Bolte-KL2007}.

\begin{definition}{[}o-minimal structure{]} \label{Def:o-minimal}
An o-minimal structure on $(\mathbb{R},+,\cdot)$ is a sequence of
boolean algebras ${\cal O}_{n}$ of ``definable'' subsets of $\mathbb{R}^{n}$,
such that for each $n\in\mathbb{N}$
\begin{enumerate}
\item[(i)] if $A$ belongs to ${\cal O}_{n}$, then $A\times\mathbb{R}$ and
$\mathbb{R}\times A$ belong to ${\cal O}_{n+1}$;
\item[(ii)] if $\Pi:\mathbb{R}^{n+1}\rightarrow\mathbb{R}^{n}$ is the canonical
projection onto $\mathbb{R}^{n}$, then for any $A$ in ${\cal O}_{n+1}$,
the set $\Pi(A)$ belongs to ${\cal O}_{n}$;
\item[(iii)] ${\cal O}_{n}$ contains the family of algebraic subsets of $\mathbb{R}^{n}$,
that is, every set of the form
\[
\{x\in\mathbb{R}^{n}:p(x)=0\},
\]
where $p:\mathbb{R}^{n}\rightarrow\mathbb{R}$ is a polynomial function.
\item[(iv)] the elements of ${\cal O}_{1}$ are exactly finite unions of intervals
and points.
\end{enumerate}
\end{definition}

Based on the definition of o-minimal structure, we can show the definition
of the \textit{definable function}.

\begin{definition}{[}Definable function{]} \label{Def:definable-function}
Given an o-minimal structure ${\cal O}$ (over $(\mathbb{R},+,\cdot)$),
a function $f:\mathbb{R}^{n}\rightarrow\mathbb{R}$ is said to be
\textit{definable} in ${\cal O}$ if its graph belongs to ${\cal O}_{n+1}$.
\end{definition}

According to \cite{vandenDries1996-GC,Bolte-KL2007}, there are some
important facts of the o-minimal structure, shown as follows.
\begin{enumerate}
\item[(i)] The collection of \textit{semialgebraic} sets is an o-minimal structure.
Recall the semialgebraic sets are Bollean combinations of sets of
the form
\[
\{x\in\mathbb{R}^{n}:p(x)=0,q_{1}(x)<0,\ldots,q_{m}(x)<0\},
\]
where $p$ and $q_{i}$'s are polynomial functions in $\mathbb{R}^{n}$.
\item[(ii)] There exists an o-minimal structure that contains the sets of the
form
\[
\{(x,t)\in[-1,1]^{n}\times\mathbb{R}:f(x)=t\}
\]
where $f$ is real-analytic around $[-1,1]^{n}$.
\item[(iii)] There exists an o-minimal structure that contains simultaneously
the graph of the exponential function $\mathbb{R}\ni x\mapsto\exp(x)$
and all semialgebraic sets.
\item[(iv)] The o-minimal structure is stable under the sum, composition, the
inf-convolution and several other classical operations of analysis.
\end{enumerate}
The Kurdyka-{\L }ojasiewicz property for the smooth definable function
and non-smooth definable function were established in \citep[Theorem 1]{Kurdyka-KL1998}
and \citep[Theorem 14]{Bolte-KL2007}, respectively. Now we are ready
to present the proof of Corollary \ref{Corollary:DL}.


\begin{proof}{[}Proof of Corollary \ref{Corollary:DL}{]} To justify
this corollary, we only need to verify the associated Lyapunov function
$F$ satisfies Kurdyka-{\L }ojasiewicz inequality. In this case
and by (\ref{Eq:Lyapunov-fun-conjugate}), $F$ can be rewritten as
follows
\begin{align*}
F({\cal W},\Gamma,{\cal G})=\alpha\left(\eL(W,\Gamma)+\frac{1}{2\nu}\|W-\Gamma\|^{2}\right)+\Omega(\Gamma)+\Omega^{*}(g)-\langle \Gamma,g\rangle.
\end{align*}
Because $\ell$ and $\sigma_{i}$'s are definable by assumptions,
then $\eL(W,\Gamma)$ are definable as compositions of definable functions.
Moreover, according to \cite{Krantz2002-real-analytic}, $\|W-\Gamma\|^{2}$
and $\langle \Gamma,g\rangle$ are semi-algebraic and thus definable. Since
the group Lasso $\Omega(\Gamma)=\sum_{g}\|\Gamma\|_{2}$ is the composition
of $\ell_{2}$ and $\ell_{1}$ norms, and the conjugate of group Lasso
penalty is the maximum of group $\ell_{2}$-norm, \emph{i.e.} $\Omega^{*}(\Gamma)=\max_{g}\|\Gamma_{g}\|_{2}$,
where the $\ell_{2}$, $\ell_{1}$, and $\ell_{\infty}$ norms are
definable, hence the group Lasso and its conjugate are definable as
compositions of definable functions. 
Therefore, $F$ is definable and hence satisfies Kurdyka-{\L }ojasiewicz
inequality by \citep[Theorem 1]{Kurdyka-KL1998}.

The verifications of other cases listed in assumptions can
be found in the proof of \citep[Proposition 1]{Zeng2019}. This finishes
the proof of this corollary. \end{proof}


\subsection{Proof of Theorem \ref{Thm:conv-SLBI+}}

Our analysis is mainly motivated by a recent paper \citep{Benning2017},
as well as the influential work \citep{Attouch2013}. According to Lemma 2.6 in
\cite{Attouch2013}, there are mainly four ingredients
in the analysis, that is, the \textit{sufficient descent property},
\textit{relative error property}, \textit{continuity property} of
the generated sequence and the \textit{Kurdyka-{\L }ojasiewicz property}
of the function. More specifically, we first establish the \textit{sufficient
descent property} of the generated sequence via exploiting the Lyapunov
function $F$ (see, (\ref{Eq:Lyapunov-fun})) in Lemma \ref{Lemma:sufficient-descent}
in Section \ref{sc:sufficient-descent}, and then show the \textit{relative
error property} of the sequence in Lemma \ref{Lemma:relative-error}
in Section \ref{sc:relative-error}. The \textit{continuity property}
is guaranteed by the continuity of $\bar{\calL}(W,\Gamma)$ and the
relation $\lim_{k\rightarrow\infty}B_{\Omega}^{g_{k}}(\Gamma_{k+1},\Gamma_{k})=0$
established in Lemma \ref{Lemma:convergence-funcvalue}(i) in Section
\ref{sc:sufficient-descent}. Thus, together with the Kurdyka-{\L }ojasiewicz
assumption of $F$, we establish the global convergence of SLBI following
by \citep[Lemma 2.6]{Attouch2013}.

Let $(\bar{W},\bar{\Gamma},\bar{g})$ be a critical point of $F$,
then the following holds
\begin{align}
 & \partial_{W}F(\bar{W},\bar{\Gamma},\bar{g})=\alpha(\nabla\eL(\bar{W})+\nu^{-1}(\bar{W}-\bar{\Gamma}))=0,\nonumber \\
 & \partial_{\Gamma}F(\bar{W},\bar{\Gamma},\bar{g})=\alpha\nu^{-1}(\bar{\Gamma}-\bar{W})+\partial\Omega(\bar{\Gamma})-\bar{g}\ni0,\label{Eq:critpoint-F}\\
 & \partial_{g}F(\bar{W},\bar{\Gamma},\bar{g})=\bar{\Gamma}-\partial\Omega^{*}(\bar{g})\ni0.\nonumber
\end{align}
By the final inclusion and the convexity of $\Omega$, it implies
$\bar{g}\in\partial\Omega(\bar{\Gamma})$. Plugging this inclusion
into the second inclusion yields $\alpha\nu^{-1}(\bar{\Gamma}-\bar{W})=0$.
Together with the first equality imples
\[
\nabla\bar{\calL}(\bar{W},\bar{\Gamma})=0,\quad\nabla\eL(\bar{W})=0.
\]
This finishes the proof of this theorem.

\subsection{Sufficient Descent Property along Lyapunov Function}

\label{sc:sufficient-descent}

Let $P_{k}:=(W_{k},\Gamma_{k})$, and $Q_{k}:=(P_{k},g_{k-1}),k\in\mathbb{N}$.
In the following, we present the sufficient descent property of $Q_{k}$
along the Lyapunov function $F$.

\noindent \textbf{Lemma.} \label{Lemma:sufficient-descent} Suppose
that $\eL$ is continuously differentiable and $\nabla\eL$ is Lipschitz
continuous with a constant $Lip>0$. Let $\{Q_{k}\}$ be a sequence
generated by SLBI with a finite initialization. If $0<\alpha<\frac{2}{\kappa(Lip+\nu^{-1})}$,
then
\[
F(Q_{k+1})\leq F(Q_{k})-\rho\|Q_{k+1}-Q_{k}\|_{2}^{2},
\]
where $\rho:=\frac{1}{\kappa}-\frac{\alpha(Lip+\nu^{-1})}{2}$.

\begin{proof} By the optimality condition of (\ref{Eq:SLBI-reform2-iter1})
and also the inclusion $p_{k}=[0,g_{k}]^{T}\in\partial R(P_{k})$,
there holds
\begin{align*}
\kappa(\alpha\nabla\bar{\calL}(P_{k})+p_{k+1}-p_{k})+P_{k+1}-P_{k}=0,
\end{align*}
which implies
\begin{align}
-\langle\alpha\nabla\bar{\calL}(P_{k}),P_{k+1}-P_{k}\rangle=\kappa^{-1}\|P_{k+1}-P_{k}\|_{2}^{2}+D(\Gamma_{k+1},\Gamma_{k})\label{Eq:innerproduct-term}
\end{align}
where
\[
D(\Gamma_{k+1},\Gamma_{k}):=\langle g_{k+1}-g_{k},\Gamma_{k+1}-\Gamma_{k}\rangle.
\]
Noting that $\bar{\calL}(P)=\eL(W)+\frac{1}{2\nu}\|W-\Gamma\|_{2}^{2}$
and by the Lipschitz continuity of $\nabla\eL(W)$ with a constant
$Lip>0$ implies $\nabla\bar{\calL}$ is Lipschitz continuous with
a constant $Lip+\nu^{-1}$. This implies
\begin{align*}
\bar{\calL}(P_{k+1})\leq\bar{\calL}(P_{k})+\langle\nabla\bar{\calL}(P_{k}),P_{k+1}-P_{k}\rangle+\frac{Lip+\nu^{-1}}{2}\|P_{k+1}-P_{k}\|_{2}^{2}.
\end{align*}
Substituting the above inequality into (\ref{Eq:innerproduct-term})
yields
\begin{align}
\alpha\bar{\calL}(P_{k+1})+D(\Gamma_{k+1},\Gamma_{k})+\rho\|P_{k+1}-P_{k}\|_{2}^{2}\leq\alpha\bar{\calL}(P_{k}).\label{Eq:sufficientdescent-barL}
\end{align}
Adding some terms in both sides of the above inequality and after
some reformulations implies
\begin{align}
 & \alpha\bar{\calL}(P_{k+1})+B_{\Omega}^{g_{k}}(\Gamma_{k+1},\Gamma_{k})\\
 & \leq\alpha\bar{\calL}(P_{k})+B_{\Omega}^{g_{k-1}}(\Gamma_{k},\Gamma_{k-1})-\rho\|P_{k+1}-P_{k}\|_{2}^{2}-\left(D(\Gamma_{k+1},\Gamma_{k})+B_{\Omega}^{g_{k-1}}(\Gamma_{k},\Gamma_{k-1})-B_{\Omega}^{g_{k}}(\Gamma_{k+1},\Gamma_{k})\right)\nonumber \\
 & =\alpha\bar{\calL}(P_{k})+B_{\Omega}^{g_{k-1}}(\Gamma_{k},\Gamma_{k-1})-\rho\|P_{k+1}-P_{k}\|_{2}^{2}-B_{\Omega}^{g_{k+1}}(\Gamma_{k},\Gamma_{k-1})-B_{\Omega}^{g_{k-1}}(\Gamma_{k},\Gamma_{k-1}),\nonumber
\end{align}
where the final equality holds for $D(\Gamma_{k+1},\Gamma_{k})-B_{\Omega}^{g_{k}}(\Gamma_{k+1},\Gamma_{k})=B_{\Omega}^{g_{k+1}}(\Gamma_{k},\Gamma_{k-1}).$
That is,
\begin{align}
F(Q_{k+1}) & \leq F(Q_{k})-\rho\|P_{k+1}-P_{k}\|_{2}^{2}-B_{\Omega}^{g_{k+1}}(\Gamma_{k},\Gamma_{k-1})-B_{\Omega}^{g_{k-1}}(\Gamma_{k},\Gamma_{k-1})\label{Eq:sufficientdescent-Breg}\\
 & \leq F(Q_{k})-\rho\|P_{k+1}-P_{k}\|_{2}^{2},\label{Eq:sufficientdescent}
\end{align}
where the final inequality holds for $B_{\Omega}^{g_{k+1}}(\Gamma_{k},\Gamma_{k-1})\geq0$
and $B_{\Omega}^{g_{k-1}}(\Gamma_{k},\Gamma_{k-1})\geq0.$ Thus, we
finish the proof of this lemma. \end{proof}

Based on Lemma \ref{Lemma:sufficient-descent}, we directly obtain
the following lemma.

\begin{lemma} \label{Lemma:convergence-funcvalue} Suppose that assumptions
of Lemma \ref{Lemma:sufficient-descent} hold. Suppose further that
Assumption \ref{Assumption} (b)-(d) hold. Then
\begin{enumerate}
\item[(i)] both $\alpha\{\bar{\calL}(P_{k})\}$ and $\{F(Q_{k})\}$ converge
to the same finite value, and $\lim_{k\rightarrow\infty}B_{\Omega}^{g_{k}}(\Gamma_{k+1},\Gamma_{k})=0.$
\item[(ii)] the sequence $\{(W_{k},\Gamma_{k},g_{k})\}$ is bounded,
\item[(iii)] $\lim_{k\rightarrow\infty}\|P_{k+1}-P_{k}\|_{2}^{2}=0$ and $\lim_{k\rightarrow\infty}D(\Gamma_{k+1},\Gamma_{k})=0,$
\item[(iv)] $\frac{1}{K}\sum_{k=0}^{K}\|P_{k+1}-P_{k}\|_{2}^{2}\rightarrow0$
at a rate of ${\cal O}(1/K)$.
\end{enumerate}
\end{lemma}

\begin{proof} By (\ref{Eq:sufficientdescent-barL}), $\bar{\calL}(P_{k})$
is monotonically decreasing due to $D(\Gamma_{k+1},\Gamma_{k})\geq0$.
Similarly, by (\ref{Eq:sufficientdescent}), $F(Q^{k})$ is also monotonically
decreasing. By the lower boundedness assumption of $\eL(W)$, both
$\bar{\calL}(P)$ and $F(Q)$ are lower bounded by their definitions,
i.e., (\ref{eq:sparse_loss}) and (\ref{Eq:Lyapunov-fun}), respectively.
Therefore, both $\{\bar{\calL}(P_{k})\}$ and $\{F(Q_{k})\}$ converge,
and it is obvious that $\lim_{k\rightarrow\infty}F(Q_{k})\geq\lim_{k\rightarrow\infty}\alpha\bar{\calL}(P_{k})$.
By (\ref{Eq:sufficientdescent-Breg}),
\begin{align*}
B_{\Omega}^{g_{k-1}}(\Gamma_{k},\Gamma_{k-1})\leq F(Q_{k})-F(Q_{k+1}),\ k=1,\ldots.
\end{align*}
By the convergence of $F(Q_{k})$ and the nonegativeness of $B_{\Omega}^{g_{k-1}}(\Gamma_{k},\Gamma_{k-1})$,
there holds
\[
\lim_{k\rightarrow\infty}B_{\Omega}^{g_{k-1}}(\Gamma_{k},\Gamma_{k-1})=0.
\]
By the definition of $F(Q_{k})=\alpha\bar{\calL}(P_{k})+B_{\Omega}^{g_{k-1}}(\Gamma_{k},\Gamma_{k-1})$
and the above equality, it yields
\[
\lim_{k\rightarrow\infty}F(Q_{k})=\lim_{k\rightarrow\infty}\alpha\bar{\calL}(P_{k}).
\]

Since $\eL(W)$ has bounded level sets, then $W_{k}$ is bounded.
By the definition of $\bar{\calL}(W,\Gamma)$ and the finiteness of
$\bar{\calL}(W_{k},\Gamma_{k})$, $\Gamma_{k}$ is also bounded due
to $W_{k}$ is bounded. The boundedness of $g_{k}$ is due to $g_{k}\in\partial\Omega(\Gamma_{k})$,
condition (d), and the boundedness of $\Gamma_{k}$.

By (\ref{Eq:sufficientdescent}), summing up (\ref{Eq:sufficientdescent})
over $k=0,1,\ldots,K$ yields
\begin{align}
\sum_{k=0}^{K}\left(\rho\|P_{k+1}-P_{k}\|^{2}+D(\Gamma_{k+1},\Gamma_{k})\right)<\alpha\bar{\calL}(P_{0})<\infty.\label{Eq:summable}
\end{align}
Letting $K\rightarrow\infty$ and noting that both $\|P_{k+1}-P_{k}\|^{2}$
and $D(\Gamma_{k+1},\Gamma_{k})$ are nonnegative, thus
\[
\lim_{k\rightarrow\infty}\|P_{k+1}-P_{k}\|^{2}=0,\quad\lim_{k\rightarrow\infty}D(\Gamma_{k+1},\Gamma_{k})=0.
\]
Again by (\ref{Eq:summable}),
\begin{align*}
\frac{1}{K}\sum_{k=0}^{K}\left(\rho\|P_{k+1}-P_{k}\|^{2}+D(\Gamma_{k+1},\Gamma_{k})\right)<K^{-1}\alpha\bar{\calL}(P_{0}),
\end{align*}
which implies $\frac{1}{K}\sum_{k=0}^{K}\|P_{k+1}-P_{k}\|^{2}\rightarrow0$
at a rate of ${\cal O}(1/K)$. \end{proof}

\subsection{Relative Error Property}

\label{sc:relative-error}

In this subsection, we provide the bound of subgradient by the discrepancy
of two successive iterates. By the definition of $F$ (\ref{Eq:Lyapunov-fun}),
\begin{align}
H_{k+1}:=\left(\begin{array}{c}
\alpha\nabla_{W}\bar{\calL}(W_{k+1},\Gamma_{k+1})\\
\alpha\nabla_{\Gamma}\bar{\calL}(W_{k+1},\Gamma_{k+1})+g_{k+1}-g_{k}\\
\Gamma_{k}-\Gamma_{k+1}
\end{array}\right)\in\partial F(Q_{k+1}),\ k\in\mathbb{N}.\label{Eq:subgradient}
\end{align}

\noindent \textbf{Lemma. } \label{Lemma:relative-error} Under assumptions
of Lemma \ref{Lemma:convergence-funcvalue}, then
\[
\|H_{k+1}\|\leq\rho_{1}\|Q_{k+1}-Q_{k}\|,\ \mathrm{for}\ H_{k+1}\in\partial F(Q_{k+1}),\ k\in\mathbb{N},
\]
where $\rho_{1}:=2\kappa^{-1}+1+\alpha(Lip+2\nu^{-1})$. Moreover,
$\frac{1}{K}\sum_{k=1}^{K}\|H_{k}\|^{2}\rightarrow0$ at a rate of
${\cal O}(1/K)$.

\begin{proof} Note that
\begin{align}
\nabla_{W}\bar{\calL}(W_{k+1},\Gamma_{k+1}) & =(\nabla_{W}\bar{\calL}(W_{k+1},\Gamma_{k+1})-\nabla_{W}\bar{\calL}(W_{k+1},\Gamma_{k}))\label{Eq:nabla-W-decomp}\\
 & +(\nabla_{W}\bar{\calL}(W_{k+1},\Gamma_{k})-\nabla_{W}\bar{\calL}(W_{k},\Gamma_{k}))+\nabla_{W}\bar{\calL}(W_{k},\Gamma_{k}).\nonumber
\end{align}
By the definition of $\bar{\calL}$ (see (\ref{eq:sparse_loss})),
\begin{align*}
\|\nabla_{W}\bar{\calL}(W_{k+1},\Gamma_{k+1})-\nabla_{W}\bar{\calL}(W_{k+1},\Gamma_{k})\| & =\nu^{-1}\|\Gamma_{k}-\Gamma_{k+1}\|,\\
\|\nabla_{W}\bar{\calL}(W_{k+1},\Gamma_{k})-\nabla_{W}\bar{\calL}(W_{k},\Gamma_{k})\| & =\|(\nabla\eL(W_{k+1})-\nabla\eL(W_{k}))+\nu^{-1}(W_{k+1}-W_{k})\|\\
 & \leq(Lip+\nu^{-1})\|W_{k+1}-W_{k}\|,
\end{align*}
where the last inequality holds for the Lipschitz continuity of $\nabla\eL$
with a constant $Lip>0$, and by (\ref{Eq:SLBI-reform-iter1}),
\begin{align*}
\|\nabla_{W}\bar{\calL}(W_{k},\Gamma_{k})\|=(\kappa\alpha)^{-1}\|W_{k+1}-W_{k}\|.
\end{align*}
Substituting the above (in)equalities into (\ref{Eq:nabla-W-decomp})
yields
\begin{align*}
\|\nabla_{W}\bar{\calL}(W_{k+1},\Gamma_{k+1})\|\leq\left[(\kappa\alpha)^{-1}+Lip+\nu^{-1}\right]\cdot\|W_{k+1}-W_{k}\|+\nu^{-1}\|\Gamma_{k+1}-\Gamma_{k}\|
\end{align*}
Thus,
\begin{align}
\|\alpha\nabla_{W}\bar{\calL}(W_{k+1},\Gamma_{k+1})\|\leq\left[\kappa^{-1}+\alpha(Lip+\nu^{-1})\right]\cdot\|W_{k+1}-W_{k}\|+\alpha\nu^{-1}\|\Gamma_{k+1}-\Gamma_{k}\|.\label{Eq:subgradbound-part1}
\end{align}

By (\ref{Eq:SLBI-reform-iter3}), it yields
\begin{align*}
g_{k+1}-g_{k}=\kappa^{-1}(\Gamma_{k}-\Gamma_{k+1})-\alpha\nabla_{\Gamma}\bar{\calL}(W_{k},\Gamma_{k}).
\end{align*}
Noting that $\nabla_{\Gamma}\bar{\calL}(W_{k},\Gamma_{k})=\nu^{-1}(\Gamma_{k}-W_{k})$,
and after some simplifications yields
\begin{align}
\label{Eq:subgradbound-part2}
\|\alpha\nabla_{\Gamma}\bar{\calL}(W_{k+1},\Gamma_{k+1})+g_{k+1}-g_{k}\| & =\|(\kappa^{-1}-\alpha\nu^{-1})\cdot(\Gamma_{k}-\Gamma_{k+1})+\alpha\nu^{-1}(W_{k}-W_{k+1})\|\nonumber \\
 & \leq\alpha\nu^{-1}\|W_{k}-W_{k+1}\|+(\kappa^{-1}-\alpha\nu^{-1})\|\Gamma_{k}-\Gamma_{k+1}\|,
\end{align}
where the last inequality holds for the triangle inequality and $\kappa^{-1}>\alpha\nu^{-1}$
by the assumption.

By (\ref{Eq:subgradbound-part1}), (\ref{Eq:subgradbound-part2}),
and the definition of $H_{k+1}$ (\ref{Eq:subgradient}), there holds
\begin{align}
\|H_{k+1}\| & \leq\left[\kappa^{-1}+\alpha(Lip+2\nu^{-1})\right]\cdot\|W_{k+1}-W_{k}\|+(\kappa^{-1}+1)\|\Gamma_{k+1}-\Gamma_{k}\|\nonumber \\
 & \leq\left[2\kappa^{-1}+1+\alpha(Lip+2\nu^{-1})\right]\cdot\|P_{k+1}-P_{k}\|\label{Eq:subgradbound-Pk}\\
 & \leq\left[2\kappa^{-1}+1+\alpha(Lip+2\nu^{-1})\right]\cdot\|Q_{k+1}-Q_{k}\|.\nonumber
\end{align}

By (\ref{Eq:subgradbound-Pk}) and Lemma \ref{Lemma:convergence-funcvalue}(iv),
$\frac{1}{K}\sum_{k=1}^{K}\|H_{k}\|^{2}\rightarrow0$ at a rate of
${\cal O}(1/K)$.

This finishes the proof of this lemma. \end{proof}

\section{Supplementary Experiments}

\subsection{Ablation Study on Image Classification}

\begin{table*}
\begin{centering}
{\footnotesize{}{}}{\footnotesize\par}
\par\end{centering}
\begin{centering}
 
\par\end{centering}
\begin{centering}
{\footnotesize{}{}}%
\begin{tabular}{c}
\begin{tabular}{c|c|cccc}
\hline 
\multicolumn{2}{c|}{{\small{}{}Dataset }} & \emph{\small{}{}{}}{\small{}{}MNIST }  & \multicolumn{1}{c}{{\small{}{}Cifar-10}} & \multicolumn{2}{c}{{\small{}{}ImageNet-2012}}\tabularnewline
\hline 
{\small{}{}Models }  & {\small{}{}Variants }  & {\small{}{}{}LeNet }  & {\small{}{}ResNet-20 }  & {\small{}{}AlexNet }  & {\small{}{}ResNet-18}\tabularnewline
\hline 
\multirow{5}{*}{\emph{\small{}{}SGD}{\small{}{} }} & {\small{}{}Naive}  & {\small{}{}{}98.87 }  & {\small{}{}86.46 }  & {\small{}{}{}--/-- }  & {\small{}{}{}60.76/79.18 }\tabularnewline
 & {\small{}{}$\mathit{l}_{1}$ }  & {\small{}{}98.52 }  & {\small{}{}67.60 }  & {\small{}{}46.49/65.45} & {\small{}{}51.49/72.45}\tabularnewline
 & {\small{}{}Mom }  & {\small{}{}99.16 }  & {\small{}{}89.44 }  & {\small{}{}55.14/78.09 }  & {\small{}{}66.98/86.97 }\tabularnewline
 & {\small{}{}Mom-W}\emph{\small{}{}d$^{\star}$}{\small{}{} }  & {\small{}{}}\textbf{\small{}99.23}{\small{} }  & {\small{}{}}\textbf{\small{}90.31}{\small{} }  & {\small{}{}56.55/79.09 }  & {\small{}{}69.76/89.18}\tabularnewline
 & {\small{}{}Nesterov }  & {\small{}{}}\textbf{\small{}99.23}{\small{} }  & {\small{}{}90.18 }  & {\small{}{}-/- }  & {\small{}{}70.19/89.30}\tabularnewline
\hline 
\multirow{5}{*}{\emph{\small{}{}Adam}{\small{}{} }} & {\small{}{}Naive}  & {\small{}{}{}99.19 }  & {\small{}{}{}89.14 }  & {\small{}{}--/-- }  & {\small{}{}59.66/83.28}\tabularnewline
 & {\small{}{}Adabound }  & {\small{}{}99.15}  & {\small{}{}87.89 }  & {\small{}{}--/-- }  & {\small{}{}--/--}\tabularnewline
 & {\small{}{}Adagrad }  & {\small{}{}99.02}  & {\small{}{}88.17 }  & {\small{}{}--/-- }  & {\small{}{}--/--}\tabularnewline
 & {\small{}{}Amsgrad }  & {\small{}{}99.14}  & {\small{}{}88.68 }  & {\small{}{}--/-- }  & {\small{}{}--/--}\tabularnewline
 & {\small{}{}Radam }  & {\small{}{}99.08}  & {\small{}{}88.44 }  & {\small{}{}--/-- }  & {\small{}{}--/--}\tabularnewline
\hline 
\multirow{3}{*}{\emph{\small{}{}}\textcolor{black}{\emph{\small{}DessiLBI}}{\small{}{} }} & {\small{}{}Naive}  & {\small{}{}99.02 }  & {\small{}{}89.26 }  & {\small{}{}55.06/77.69 }  & {\small{}{}65.26/86.57 }\tabularnewline
 & {\small{}{}Mom }  & {\small{}{}{}99.19 }  & {\small{}{}{}89.72 }  & {\small{}{}56.23/78.48 }  & {\small{}{}68.55/87.85}\tabularnewline
 & {\small{}{}Mom-Wd }  & {\small{}{}{}99.20 }  & {\small{}{}{}89.95 }  & \textbf{\small{}{}57.09/79.86 }{\small{} } & \textbf{\small{}{}{}{}70.55/89.56}\tabularnewline
\cline{1-6} \cline{3-6} \cline{4-6} \cline{5-6} \cline{6-6} 
\end{tabular}\tabularnewline
\end{tabular}
\par\end{centering}
\begin{centering}
 
\par\end{centering}
{\small{}{}\caption{\label{table:supervised_imagenet_mnist_cifar} Top-1/Top-5 accuracy(\%) on ImageNet-2012
and test accuracy on MNIST/Cifar-10. $^{\star}$: results from the
official pytorch website. We use the official pytorch codes to run
the competitors. All models are trained by 100 epochs. In this table,
we run the experiment by ourselves except for SGD Mom-Wd on ImageNet
which is reported in https://pytorch.org/docs/stable/torchvision/models.html. }
} 
\end{table*}

\textbf{Experimental Design}. We compare different variants of SGD and Adam
in the experiments. By default, the learning rate of competitors is
set as $0.1$ for SGD and its variant and $0.001$ for Adam and its
variants, and gradually decreased by 1/10 every 30 epochs. In particular,
we have,

SGD: (1) Naive SGD: the standard SGD with batch input. (2) SGD with
$\mathit{l}_{1}$ penalty (Lasso). The $\mathit{l}_{1}$ norm is applied to
penalize the weights of SGD by encouraging the sparsity of learned
model, with the regularization parameter of the $\mathit{l}_{1}$ penalty term being set as 
$1e^{-3}$ 
(3) SGD with momentum (Mom): we utilize momentum 0.9 in SGD. (4) SGD
with momentum and weight decay (Mom-Wd): we set the momentum 0.9 and
the standard $\mathit{l}_{2}$ weight decay with the coefficient weight
$1e^{-4}$. (5) SGD with Nesterov (Nesterov): the SGD uses nesterov
momentum 0.9.

Adam: (1) Naive Adam: it refers to the standard version of Adam. We
report the results of several recent variants of Adam, including (2) Adabound,
(3) Adagrad, (4) Amsgrad, and (5) Radam.

The results of image classification are shown in Tab.~ \ref{table:supervised_imagenet_mnist_cifar} . It  shows the experimental results on ImageNet-2012, Cifar-10, and MNIST of some classical networks -{}- LeNet, AlexNet and ResNet. 
 Our DessiLBI  variants may achieve comparable or even better performance than SGD variants in 100 epochs, indicating the efficacy in learning dense, over-parameterized models. The visualization of learned ResNet-18 on ImageNet-2012 is given in Fig. \ref{fig:imagenet-vis}.

\begin{figure}[htb]
	\centering%
	\begin{tabular}{c}
		\hspace{-0.2in}\includegraphics[width=6in]{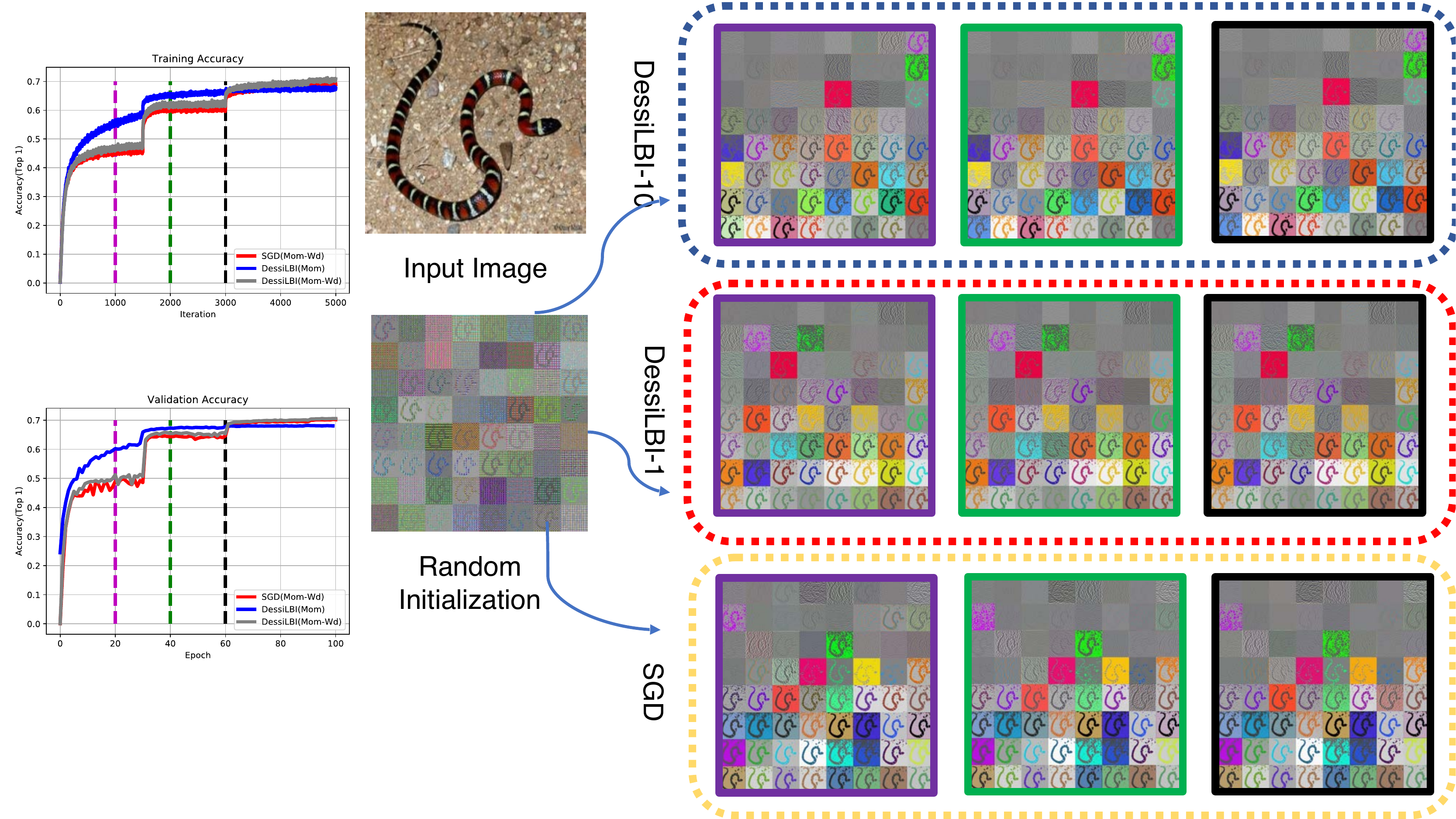}\tabularnewline
	\end{tabular}
	\caption{\label{fig:imagenet-vis} Visualization of the first convolutional layer filters of ResNet-18 trained on ImageNet-2012. Given the input
		image and initial weights visualized in the middle, filter response gradients
		at 20 (purple), 40 (green), and 60 (black) epochs are visualized by \cite{springenberg2014striving}. }  
\end{figure}

\subsection{Ablation Study of VGG16 and ResNet56 on Cifar10 \label{subsec:SplitLBI-in-training-2}}
To further study the influence of hyperparameters, we record performance
of $W_t$ for each epoch $t$ with different combinations of hyperparameters. The experiments is conducted
5 times each, we show the mean in the table, the standard error can
be found in the corresponding figure. We perform experiments on Cifar10
and two commonly used network VGG16 and ResNet56.

\begin{figure}
\centering
 \includegraphics[width=6.5in]{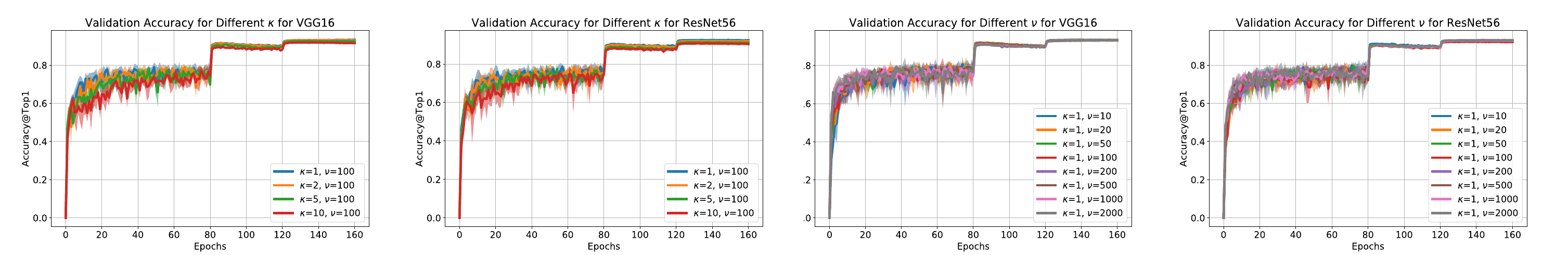}
  \caption{Validation curves of dense models $W_t$ for different $\kappa$ and $\nu$.   \label{figure:ablation} For DessiLBI we find that the model accuracy is robust to the hyperparameters both in terms of convergence rate and generalization ability. Here validation accuracy means the accuracy on test set of Cifar10. The first one is the result for VGG16 ablation study on $\kappa$, the second one is the result for ResNet56 ablation study on $\kappa$,  the third one is the result for VGG16 ablation study on $\nu$ and the forth one is the result for ResNet56 ablation study on $\nu$.  }
\end{figure}

On $\kappa$ , we keep $\nu=100$ and try $\kappa=1,2,5,10$, the
validation curves of models $W_t$ are shown in Fig. \ref{figure:ablation} and Table \ref{table:ablation_k} summarizes the mean accuracies. Table \ref{table:sparsek} summarizes best validation accuracies achieved at some epochs, together with their sparsity rates. These results show that larger kappa leads to slightly lower validation accuracies, where
the numerical results are shown in Table \ref{table:ablation_k} . We can find that $\kappa=1$
achieves the best test accuracy.

\begin{table}
\begin{centering}
\begin{tabular}{c|c|c|c|c|c|c}
\hline 
Type & Model & $\kappa=1$ & $\kappa=2$ & $\kappa=5$ & $\kappa=10$ & SGD\tabularnewline
\hline 
\multirow{2}{*}{Full} & Vgg16 & \textcolor{black}{93.46} & \textcolor{black}{93.27} & \textcolor{black}{92.77} & \textcolor{black}{92.03} & 93.57\tabularnewline
\cline{2-7} \cline{3-7} \cline{4-7} \cline{5-7} \cline{6-7} \cline{7-7} 
 & ResNet56 & \textcolor{black}{92.71} & \textcolor{black}{92.18} & \textcolor{black}{91.50} & \textcolor{black}{90.92} & 93.08\tabularnewline
\hline 
\multirow{2}{*}{Sparse} & Vgg16 & \textcolor{black}{93.31} & \textcolor{black}{93.00} & \textcolor{black}{92.36} & \textcolor{black}{76.25} & -\tabularnewline
\cline{2-7} \cline{3-7} \cline{4-7} \cline{5-7} \cline{6-7} \cline{7-7} 
 & ResNet56 & \textcolor{black}{92.37} & \textcolor{black}{91.85} & \textcolor{black}{89.48} & \textcolor{black}{87.02} & -\tabularnewline
\hline 
\end{tabular}
\par\end{centering}
\caption{This table shows results for different $\kappa$, the results are all the best test accuracy.  Here we test two widely-used models: VGG16 and ResNet56 on Cifar10.  For results in this table, we keep $\nu=100$. Full means that we use the trained model weights directly, Sparse means the model weights are combined with mask generated by $\Gamma$ support. Sparse result has no finetuning process, the result is comparable to its Full counterpart. For this experiment, we propose that $\kappa=1$ is a good choice. For all the model, we train for 160 epochs with initial learning rate (lr) of 0. 1 and decrease by 0.1 at  epoch 80 and 120.\label{table:ablation_k}}
\end{table}

\begin{table}
\begin{centering}
\begin{tabular}{c|c|c|c|c|c|c|c|c|c}
\hline 
\multicolumn{2}{c|}{{\small{}Model}} & \multicolumn{2}{c|}{{\small{}Ep20}} & \multicolumn{2}{c|}{{\small{}Ep40}} & \multicolumn{2}{c|}{{\small{}Ep80}} & \multicolumn{2}{c}{{\small{}Ep160}}\tabularnewline
\hline 
\multirow{5}{*}{{\small{}Vgg16}} & {\small{}Term} & {\small{}Sparsity} & {\small{}Acc} & {\small{}Spasity} & {\small{}Acc} & {\small{}Spasity} & {\small{}Acc} & {\small{}Spasity} & {\small{}Acc}\tabularnewline
\cline{2-10} \cline{3-10} \cline{4-10} \cline{5-10} \cline{6-10} \cline{7-10} \cline{8-10} \cline{9-10} \cline{10-10} 
 & {\small{}$\kappa=1$} & 96.62 & 71.51 & 96.62 & 76.92 & 96.63 & 77.48 & 96.63 & 93.31\tabularnewline
\cline{2-10} \cline{3-10} \cline{4-10} \cline{5-10} \cline{6-10} \cline{7-10} \cline{8-10} \cline{9-10} \cline{10-10} 
 & {\small{}$\kappa=2$} & 51.86 & 72.98 & 71.99 & 73.64 & 75.69 & 74.54 & 75.72 & 93.00\tabularnewline
\cline{2-10} \cline{3-10} \cline{4-10} \cline{5-10} \cline{6-10} \cline{7-10} \cline{8-10} \cline{9-10} \cline{10-10} 
 & {\small{}$\kappa=5$} & 8.19 & 10.00 & 17.64 & 34.25 & 29.76 & 69.92 & 30.03 & 92.36\tabularnewline
\cline{2-10} \cline{3-10} \cline{4-10} \cline{5-10} \cline{6-10} \cline{7-10} \cline{8-10} \cline{9-10} \cline{10-10} 
 & {\small{}$\kappa=10$} & 0.85 & 10.00 & 6.62 & 10.00 & 12.95 & 38.38 & 13.26 & 76.25\tabularnewline
\hline 
\multirow{5}{*}{{\small{}ResNet56}} & {\small{}Term} & {\small{}Sparsity} & {\small{}Acc} & {\small{}Spasity} & {\small{}Acc} & {\small{}Spasity} & {\small{}Acc} & {\small{}Spasity} & {\small{}Acc}\tabularnewline
\cline{2-10} \cline{3-10} \cline{4-10} \cline{5-10} \cline{6-10} \cline{7-10} \cline{8-10} \cline{9-10} \cline{10-10} 
 & {\small{}$\kappa=1$} & 96.79 & 73.50 & 96.87 & 75.27 & 96.69 & 77.47 & 99.68 & 92.37\tabularnewline
\cline{2-10} \cline{3-10} \cline{4-10} \cline{5-10} \cline{6-10} \cline{7-10} \cline{8-10} \cline{9-10} \cline{10-10} 
 & {\small{}$\kappa=2$} & 76.21 & 72.85 & 81.41 & 74.72 & 84.17 & 75.64 & 84.30 & 91.85\tabularnewline
\cline{2-10} \cline{3-10} \cline{4-10} \cline{5-10} \cline{6-10} \cline{7-10} \cline{8-10} \cline{9-10} \cline{10-10} 
 & {\small{}$\kappa=5$} & 36.58 & 60.43 & 53.07 & 76.00 & 57.48 & 75.67 & 57.74 & 89.48\tabularnewline
\cline{2-10} \cline{3-10} \cline{4-10} \cline{5-10} \cline{6-10} \cline{7-10} \cline{8-10} \cline{9-10} \cline{10-10} 
 & {\small{}$\kappa=10$} & 3.12 & 10.20 & 29.43 & 53.36 & 41.18 & 74.56 & 41.14 & 87.02\tabularnewline
\hline 
\end{tabular}
\par\end{centering}
\caption{Sparsity rate and validation accuracy for different $\kappa$ at different epochs. \label{table:sparsek} Here we pick the test accuracy for specific epoch. In this experiment, we keep $\nu=100$. We pick epoch 20, 40, 80 and 160 to show the growth of sparsity and sparse model accuracy. Here Sparsity is defined in Sec. \ref{sec:exp}, and Acc means the test accuracy for sparse model. A sparse model is a model at designated epoch $t$ combined with the mask as the support of $\Gamma_t$.}
\end{table}

On $\nu$ , we keep $\kappa=1$ and try $\nu=10,20,50,100,200,500,1000,2000$
the validation curve and mean accuracies are show in Fig. \ref{figure:ablation}  and Table \ref{table:ablation_n}. Table \ref{table:sparsen} summarizes best validation accuracies achieved at some epochs, together with their sparsity rates. By carefully tuning
$\nu$ we can achieve similar or even better results compared to SGD.
Different from $\kappa$, $\nu$ has less effect on the generalization
performance. By tuning it carefully, we can even get a sparse model
with slightly better performance than SGD trained model.

\begin{table}
\begin{centering}
\begin{tabular}{c|c|c|c|c|c|c|c|c|c|c}
\hline 
{\small{}Type} & {\small{}Model} & {\small{}$\nu=10$} & {\small{}$\nu=20$} & {\small{}$\nu=50$} & {\small{}$\nu=100$} & {\small{}$\nu=200$} & {\small{}$\nu=500$} & {\small{}$\nu=1000$} & {\small{}$\nu=2000$} & SGD\tabularnewline
\hline 
\multirow{2}{*}{{\small{}Full}} & {\small{}Vgg16} & 93.66 & 93.59 & 93.57 & 93.39 & 93.38 & 93.35 & 93.43 & 93.46 & 93.57\tabularnewline
\cline{2-11} \cline{3-11} \cline{4-11} \cline{5-11} \cline{6-11} \cline{7-11} \cline{8-11} \cline{9-11} \cline{10-11} \cline{11-11} 
 & {\small{}ResNet56} & 93.12 & 92.68 & 92.78 & 92.45 & 92.95 & 93.11 & 93.16 & 93.31 & 93.08\tabularnewline
\hline 
\multirow{2}{*}{{\small{}Sparse}} & {\small{}Vgg16} & 93.39 & 93.42 & 93.39 & 93.23 & 93.21 & 93.01 & 92.68 & 10 & -\tabularnewline
\cline{2-11} \cline{3-11} \cline{4-11} \cline{5-11} \cline{6-11} \cline{7-11} \cline{8-11} \cline{9-11} \cline{10-11} \cline{11-11} 
 & {\small{}ResNet56} & 92.81 & 92.19 & 92.40 & 92.10 & 92.68 & 92.81 & 92.84 & 88.96 & -\tabularnewline
\hline 
\end{tabular}
\par\end{centering}
\caption{Results for different $\nu$, the results are all the best test accuracy. Here we test two widely-used model : VGG16 and ResNet56 on Cifar10.  For results in this table, we keep $\kappa=1$. Full means that we use the trained model weights directly, Sparse means the model weights are combined with mask generated by $\Gamma$ support. Sparse result has no finetuning process, the result is comparable to its Full counterpart. For all the model, we train for 160 epochs with initial learning rate (lr) of 0.1 and decrease by 0.1 at  epoch 80 and 120.  \label{table:ablation_n}}
\end{table}

\begin{table}
\begin{centering}
\begin{tabular}{c|c|c|c|c|c|c|c|c|c}
\hline 
\multicolumn{2}{c|}{{\small{}Model}} & \multicolumn{2}{c|}{{\small{}Ep20}} & \multicolumn{2}{c|}{{\small{}Ep40}} & \multicolumn{2}{c|}{{\small{}Ep80}} & \multicolumn{2}{c}{{\small{}Ep160}}\tabularnewline
\hline 
\multirow{9}{*}{{\small{}Vgg16}} & {\small{}Term} & {\small{}Sparsity} & {\small{}Acc} & {\small{}Spasity} & {\small{}Acc} & {\small{}Spasity} & {\small{}Acc} & {\small{}Spasity} & {\small{}Acc}\tabularnewline
\cline{2-10} \cline{3-10} \cline{4-10} \cline{5-10} \cline{6-10} \cline{7-10} \cline{8-10} \cline{9-10} \cline{10-10} 
 & {\small{}$\nu=10$} & 96.64 & 71.07 & 96.64 & 77.70 & 96.65 & 79.46 & 96.65 & 93.34\tabularnewline
\cline{2-10} \cline{3-10} \cline{4-10} \cline{5-10} \cline{6-10} \cline{7-10} \cline{8-10} \cline{9-10} \cline{10-10} 
 & {\small{}$\nu=20$} & 96.64 & 69.11 & 96.64 & 77.63 & 96.65 & 77.08 & 96.65 & 93.42\tabularnewline
\cline{2-10} \cline{3-10} \cline{4-10} \cline{5-10} \cline{6-10} \cline{7-10} \cline{8-10} \cline{9-10} \cline{10-10} 
 & {\small{}$\nu=50$} & 96.64 & 74.91 & 96.65 & 74.21 & 96.65 & 79.15 & 96.65 & 93.38\tabularnewline
\cline{2-10} \cline{3-10} \cline{4-10} \cline{5-10} \cline{6-10} \cline{7-10} \cline{8-10} \cline{9-10} \cline{10-10} 
 & {\small{}$\nu=100$} & 96.64 & 74.82 & 96.64 & 73.22 & 96.64 & 78.09 & 96.64 & 93.23\tabularnewline
\cline{2-10} \cline{3-10} \cline{4-10} \cline{5-10} \cline{6-10} \cline{7-10} \cline{8-10} \cline{9-10} \cline{10-10} 
 & {\small{}$\nu=200$} & 91.69 & 73.67 & 94.06 & 74.67 & 94.15 & 75.20 & 94.15 & 93.21\tabularnewline
\cline{2-10} \cline{3-10} \cline{4-10} \cline{5-10} \cline{6-10} \cline{7-10} \cline{8-10} \cline{9-10} \cline{10-10} 
 & {\small{}$\nu=500$} & 18.20 & 10.00 & 59.94 & 67.88 & 82.03 & 78.69 & 82.32 & 93.01\tabularnewline
\cline{2-10} \cline{3-10} \cline{4-10} \cline{5-10} \cline{6-10} \cline{7-10} \cline{8-10} \cline{9-10} \cline{10-10} 
 & {\small{}$\nu=1000$} & 6.43 & 10.00 & 17.88 & 10.00 & 49.75 & 61.31 & 51.21 & 92.68\tabularnewline
\cline{2-10} \cline{3-10} \cline{4-10} \cline{5-10} \cline{6-10} \cline{7-10} \cline{8-10} \cline{9-10} \cline{10-10} 
 & {\small{}$\nu=2000$} & 0.22 & 10.00 & 6.89 & 10.00 & 18.15 & 10.00 & 19.00 & 10.00\tabularnewline
\hline 
\multirow{9}{*}{{\small{}ResNet56}} & {\small{}Term} & {\small{}Sparsity} & {\small{}Acc} & {\small{}Spasity} & {\small{}Acc} & {\small{}Spasity} & {\small{}Acc} & {\small{}Spasity} & {\small{}Acc}\tabularnewline
\cline{2-10} \cline{3-10} \cline{4-10} \cline{5-10} \cline{6-10} \cline{7-10} \cline{8-10} \cline{9-10} \cline{10-10} 
 & {\small{}$\nu=10$} & 99.97 & 73.37 & 99.95 & 71.64 & 99.74 & 76.46 & 99.74 & 92.81\tabularnewline
\cline{2-10} \cline{3-10} \cline{4-10} \cline{5-10} \cline{6-10} \cline{7-10} \cline{8-10} \cline{9-10} \cline{10-10} 
 & {\small{}$\nu=20$} & 99.97 & 72.58 & 99.84 & 74.16 & 99.69 & 72.37 & 99.72 & 92.19\tabularnewline
\cline{2-10} \cline{3-10} \cline{4-10} \cline{5-10} \cline{6-10} \cline{7-10} \cline{8-10} \cline{9-10} \cline{10-10} 
 & {\small{}$\nu=50$} & 99.96 & 70.72 & 99.89 & 73.96 & 99.79 & 74.93 & 99.77 & 92.40\tabularnewline
\cline{2-10} \cline{3-10} \cline{4-10} \cline{5-10} \cline{6-10} \cline{7-10} \cline{8-10} \cline{9-10} \cline{10-10} 
 & {\small{}$\nu=100$} & 96.31 & 73.63 & 96.63 & 75.79 & 96.55 & 72.94 & 96.57 & 92.10\tabularnewline
\cline{2-10} \cline{3-10} \cline{4-10} \cline{5-10} \cline{6-10} \cline{7-10} \cline{8-10} \cline{9-10} \cline{10-10} 
 & {\small{}$\nu=200$} & 91.98 & 75.30 & 94.38 & 72.13 & 94.87 & 73.75 & 94.88 & 92.68\tabularnewline
\cline{2-10} \cline{3-10} \cline{4-10} \cline{5-10} \cline{6-10} \cline{7-10} \cline{8-10} \cline{9-10} \cline{10-10} 
 & {\small{}$\nu=500$} & 74.44 & 65.58 & 90.00 & 74.12 & 92.96 & 71.91 & 92.99 & 92.81\tabularnewline
\cline{2-10} \cline{3-10} \cline{4-10} \cline{5-10} \cline{6-10} \cline{7-10} \cline{8-10} \cline{9-10} \cline{10-10} 
 & {\small{}$\nu=1000$} & 24.32 & 10.85 & 75.68 & 70.23 & 88.56 & 79.67 & 88.80 & 92.48\tabularnewline
\cline{2-10} \cline{3-10} \cline{4-10} \cline{5-10} \cline{6-10} \cline{7-10} \cline{8-10} \cline{9-10} \cline{10-10} 
 & {\small{}$\nu=2000$} & 0.65 & 10.00 & 26.66 & 13.30 & 74.98 & 70.38 & 75.92 & 88.95\tabularnewline
\hline 
\end{tabular}
\par\end{centering}
\caption{Sparsity rate and validation accuracy for different $\nu$ at different epochs. Here we pick the test accuracy for specific epoch. In this experiment, we keep $\kappa=1$. We pick epoch 20, 40, 80 and 160 to show the growth of sparsity and sparse model accuracy. Here Sparsity is defined in Sec. \ref{sec:exp} as the percentage of nonzero parameters, and Acc means the test accuracy for sparse model. A sparse model is a model at designated epoch $t$ combined with mask as the support of $\Gamma_t$. \label{table:sparsen}}
\end{table}

\section{Computational Cost of DessiLBI}

We further compare the computational cost of
different optimizers: SGD (Mom), DessiLBI (Mom) and Adam (Naive). We
test each optimizer on one GPU, and all the experiments are done on
one GTX2080. For computational cost, we judge them from two aspects : GPU memory usage and time needed for one batch. 
The batch size here is 64, experiment is performed on VGG-16 as shown in Table \ref{table:ccost}.

%
%

\begin{table}
\centering
\begin{tabular}{c| ccc}
\hline
optimizer & SGD &DessiLBI &Adam\\
\hline
Mean Batch Time & 0.0197 &0.0221 &0.0210 \\
\hline
GPU Memory & 1161MB &1459MB &1267MB\\
\hline
\end{tabular}
\caption{Computational and Memory Costs.\label{table:ccost}} 
\end{table}

\section{Fine-tuning of sparse subnetworks} 
\label{sec:sparse}

\begin{table}
\centering
\begin{tabular}{c|c|c|c}
\hline
Layer & FC1 & FC2 &FC3 \\
\hline
Sparsity &0.049 & 0.087 & 0.398 \\
Number of Weights & 235200 & 30000& 1000\\
\hline
\end{tabular}
\caption{This table shows the sparsity for every layer of Lenet-3. Here sparsity is defined in Sec. \ref{sec:exp}, number of weights denotes the total number of parameters in the designated layer. It is interesting that the $\Gamma$ tends to put lower sparsity on layer with more parameters. \label{table:layer sparsity_lenet}}
\end{table}

\begin{table}
\centering
\begin{tabular}{c|c|c|c|c|c}
\hline
Layer & Conv1 & Conv2 &FC1 & FC2&FC3 \\
\hline
Sparsity &0.9375& 1 &0.0067 & 0.0284 &0.1551 \\
Number of Weights & 576 & 36864& 3211264 & 65536 & 2560\\
\hline
\end{tabular}
\caption{This table shows the sparsity for every layer of Conv-2. Here sparsity is defined in Sec. \ref{sec:exp}, number of weights denotes the total number of parameters in the designated layer. The sparsity is more significant in fully connected (FC) layers than convolutional layers. \label{table:layer sparsity_conv_2}}
\end{table}

\begin{table}
\centering
\begin{tabular}{c|c|c|c|c|c|c|c}
\hline
Layer &  Conv1 & Conv2 &  Conv3 & Conv4 &FC1 & FC2 &FC3 \\
\hline
Sparsity &0.921875& 1 & 1 & 1 &  0.0040 & 0.0094 & 0.1004\\
Number of Weights & 576 & 36864& 73728 & 147456 & 1605632 &65536 & 2560 \\
\hline
\end{tabular}
\caption{This table shows the sparsity for every layer of Conv-4. Here sparsity is defined in Sec. \ref{sec:exp}, number of weights denotes the total number of parameters in the designated layer. Most of the convolutional layers are kept while the FC layers are very sparse. \label{table:layer sparsity_conv_4} }
\end{table}

We design the experiment on MNIST, inspired by \cite{lottery-iclr19}. Here, we explore the subnet obtained by $\Gamma_T$ after $T=100$ epochs of training. As in \cite{frankle2019lottery}, we adopt the ``rewind'' trick: re-loading the subnet mask of $\Gamma_{100}$ at different epochs, followed by fine-tuning.
In particular, along the training paths, we reload the subnet models at Epoch 0, Epoch 30, 60, 90, and 100, and further fine-tune these models by DessiLBI (Mom-Wd). All the models use the same initialization and hence the subnet model at Epoch 0 gives the retraining with the same random initialization as proposed to find winning tickets of lottery in \cite{lottery-iclr19}. 
We will denote the rewinded fine-tuned model at epoch 0 as (Lottery), and those at epoch 30, 60, 90, and 100, as F-epoch30,
F-epoch60, F-epoch90, and F-epoch100, respectively. 
Three networks are studied here 
-- LeNet-3, Conv-2, and Conv-4. LeNet-3 removes one convolutional
layer of LeNet-5; and it is thus less over-parameterized than the
other two networks. Conv-2 and Conv-4, as the scaled-down variants
of VGG family as done in \cite{lottery-iclr19}, have two and
four fully-connected layers, respectively, followed by max-pooling after every two
convolutional layer.

The whole sparsity for Lenet-3 is 0.055, Conv-2 is 0.0185, and Conv-4 is 0.1378. Detailed sparsity for every layer of the model is shown in Table \ref{table:layer sparsity_lenet}, \ref{table:layer sparsity_conv_2}, \ref{table:layer sparsity_conv_4}. We find that fc-layers are sparser than conv-layers. 

We compare DessiLBI variants to the SGD (Mom-Wd) and SGD (Lottery) \citep{lottery-iclr19} in the same structural sparsity
and the results are shown in Fig. \ref{fig:finetunemnist}. In this exploratory experiment, one can see that for overparameterized networks -- Conv-2 and Conv-4, fine-tuned rewinding  subnets -- F-epoch30, F-epoch60, F-epoch90, and F-epoch100, can produce \emph{better}
results than the full models; while for the less over-parameterized model LeNet-3, fine-tuned subnets may achieve less yet still comparable performance to the dense models and remarkably better than the retrained sparse subnets from beginning (i.e. DessiLBI/SGD (Lottery)). These phenomena suggest that the subnet architecture disclosed by structural sparsity parameter $\Gamma_T$ is valuable, for fine-tuning sparse models with comparable or even better performance than the dense models of $W_T$. 

\begin{figure}
\centering 
\includegraphics[scale=0.65]{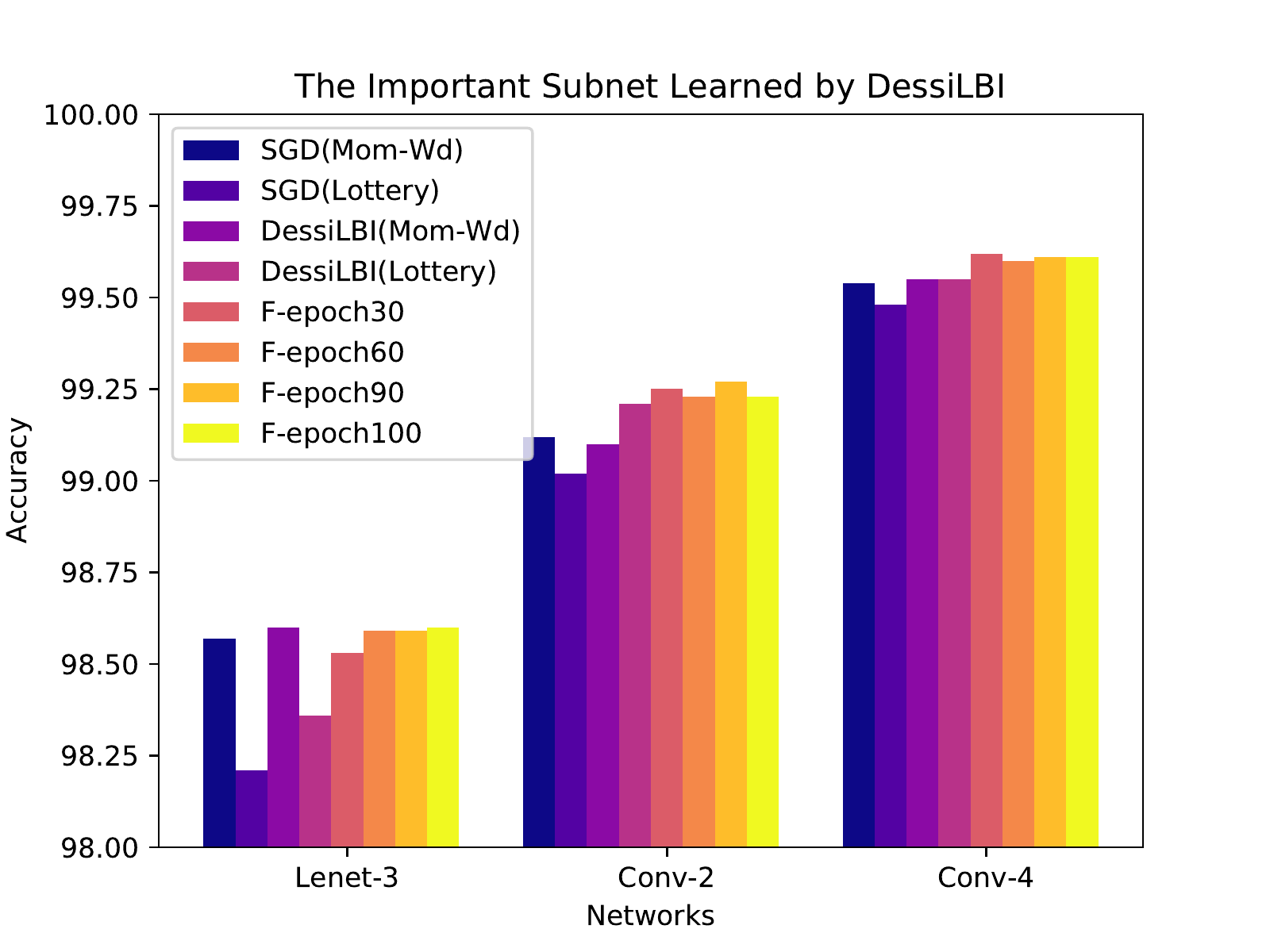} 
\caption{\label{fig:finetunemnist} Fine-tuning of sparse subnets learned by DessiLBI may achieve comparable or better performance than dense models. F-epoch$k$ indicates
the fine-tuned model comes from the Epoch $k$. DessiLBI (Lottery) and SGD (Lottery) use the same sparsity rate for each layer and the same initialization for 
retrain.}
\end{figure}

\section{Retraining of sparse subnets found by DessiLBI (Lottery)}
Here we provide more details on the experiments in Fig. \ref{figure:lotteryticketacc}. Table \ref{table:lotterysetting} gives the details on hyper-parameter setting. Moreover, Figure \ref{figure:lotteryticketsparsity} provides the sparsity variations during DessiLBI training in Fig. \ref{figure:lotteryticketacc}. 

\begin{table}
\begin{centering}
{\small{}{}}%
\begin{tabular}{l|lllllllll}
\hline
Network   & Penalty   & Optimizer  & $\alpha$ & $\nu$ & $\kappa$ & $\lambda$ & Momentum & Nesterov \\ \hline
VGG-16    & Group Lasso & DessiLBI         & 0.1      & 100   & 1        & 0.1       & 0.9      & Yes      \\ \hline
ResNet-56 & Group Lasso & DessiLBI         & 0.1      & 100   & 1        & 0.05      & 0.9      & Yes      \\ \hline
VGG-16(Lasso)    & Lasso & DessiLBI         & 0.1      & 500   & 1        & 0.05      & 0.9      & Yes      \\ \hline
ResNet-50(Lasso) & Lasso & DessiLBI         & 0.1      & 200   & 1        & 0.03      & 0.9      & Yes      \\ \hline

\end{tabular}
\par\end{centering}
\begin{centering}
 
\par\end{centering}
\centering{}{\small{}{}\caption{\label{table:lotterysetting}Hyperparameter setting 
for the experiments in Figure \ref{figure:lotteryticketacc}.}
} 
\end{table}

\begin{figure}
\centering{}%
\begin{tabular}{cccc}
 \includegraphics[width=1.5in]{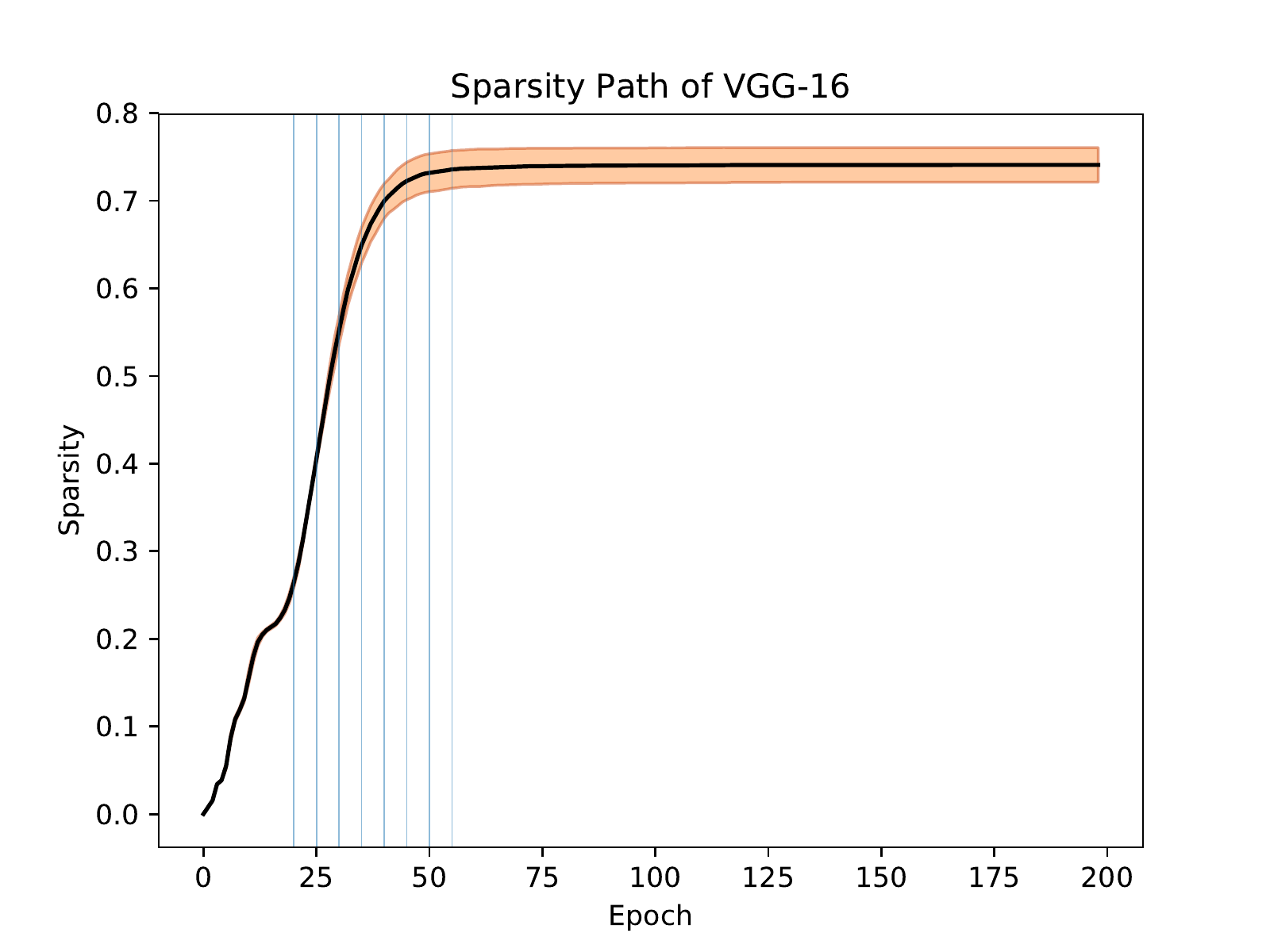} & 
 \includegraphics[width=1.5in]{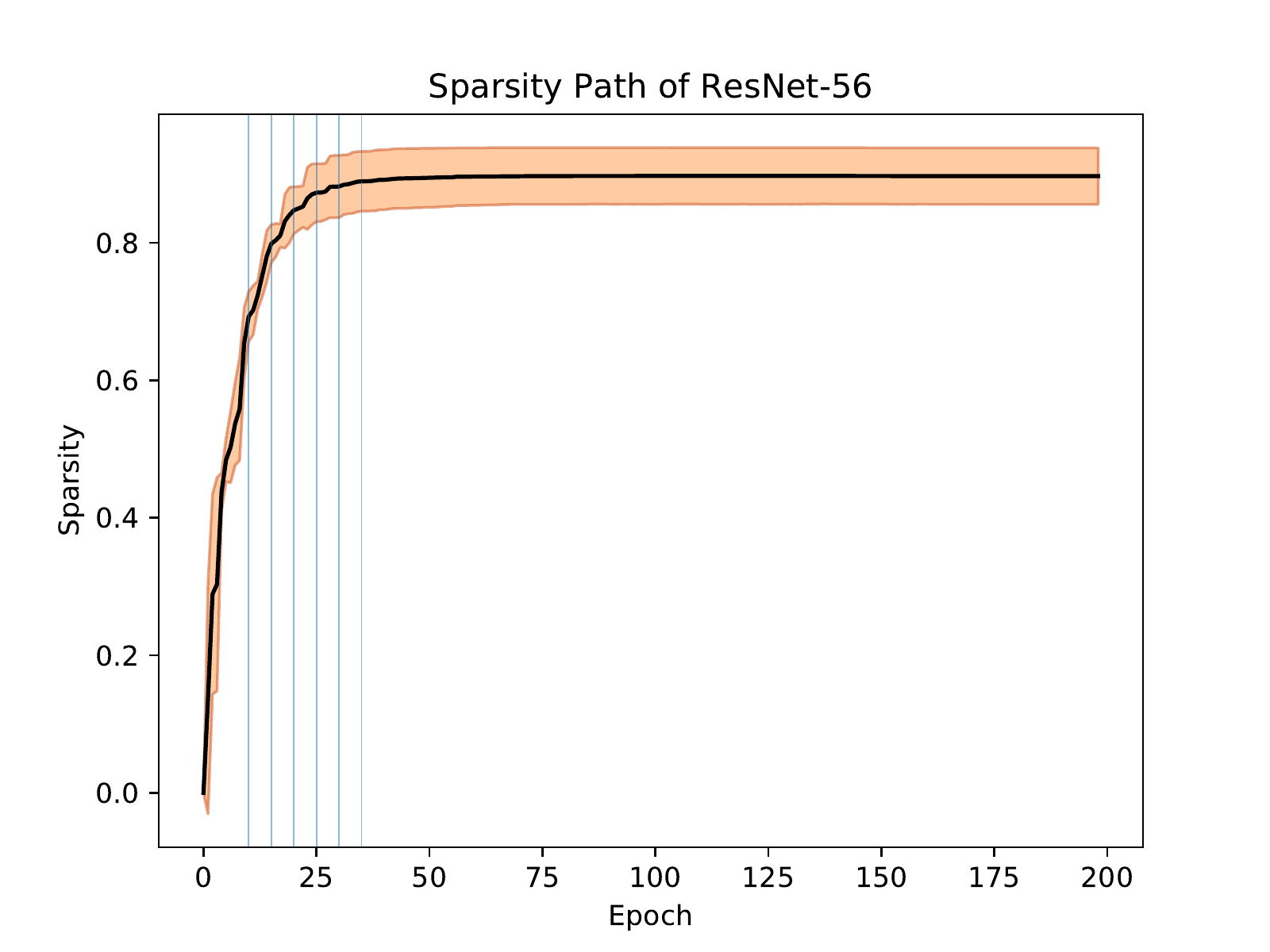} &
 \includegraphics[width=1.5in]{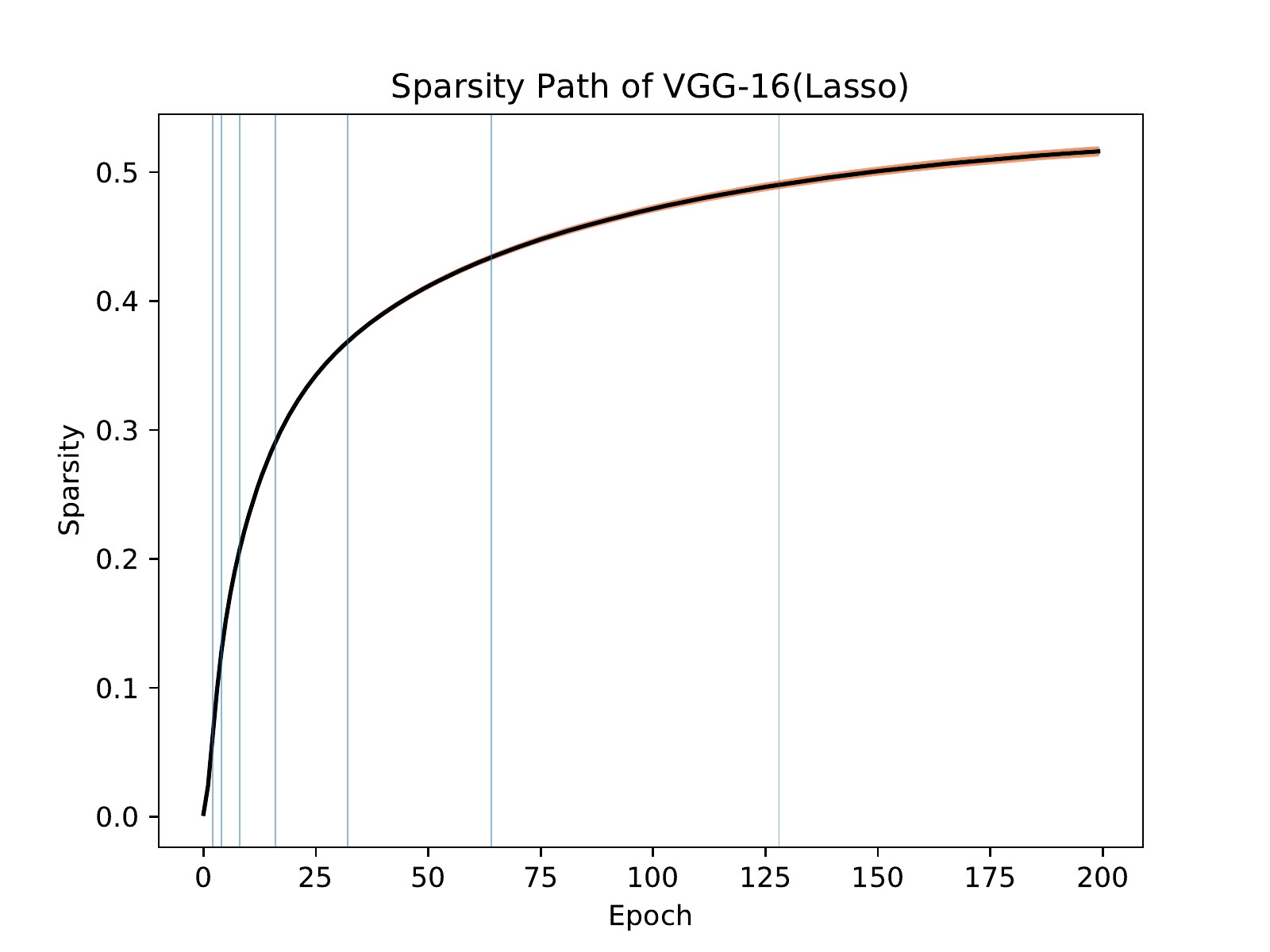} & 
 \includegraphics[width=1.5in]{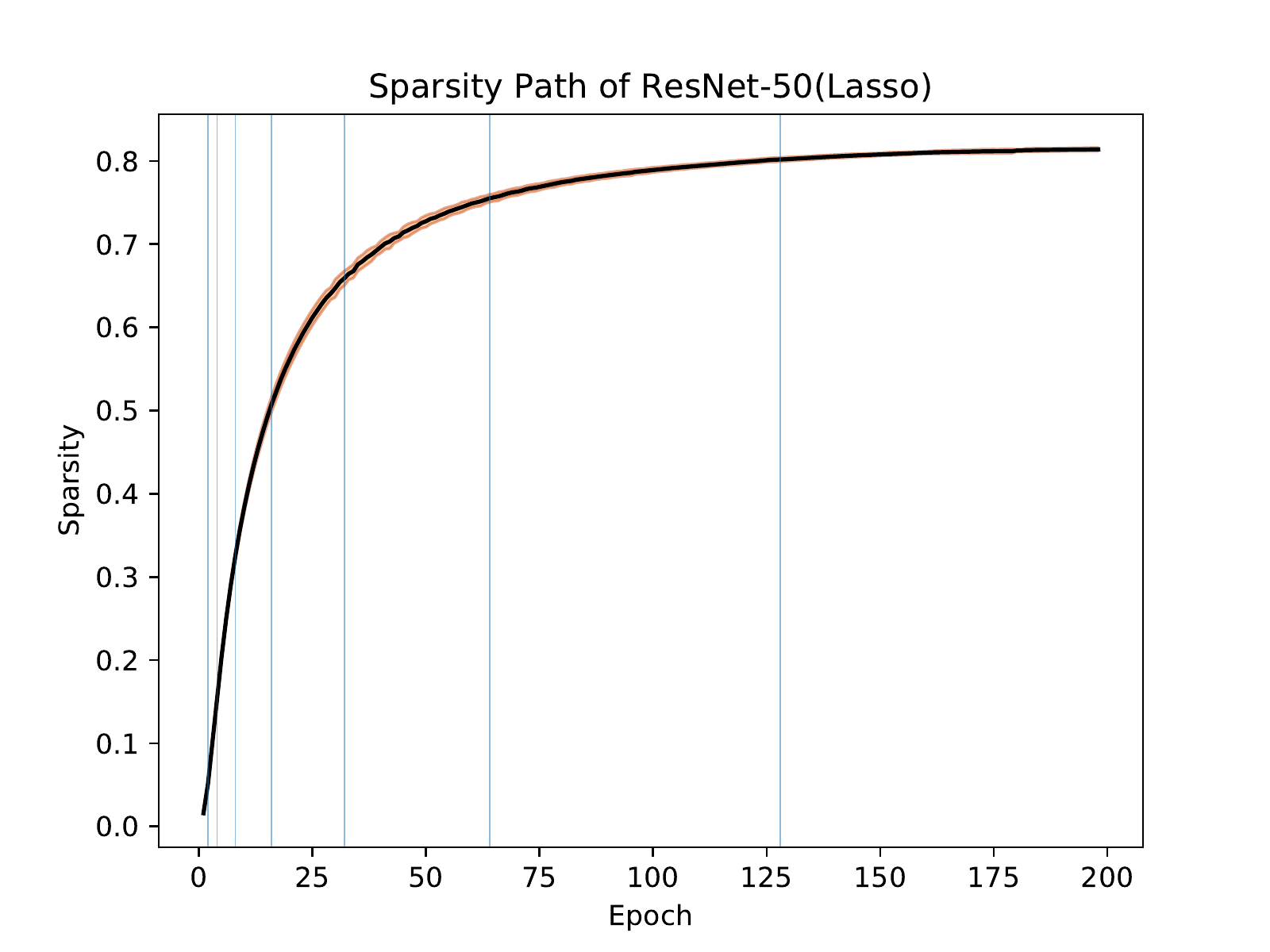} 
\tabularnewline
(a) VGG-16 & (b) ResNet-56 & (c) VGG-16 (Lasso) & (d) ResNet-50 (Lasso) \tabularnewline
\end{tabular} \caption{Sparsity changing during training process of DessiLBI (Lottery) for VGG and ResNets (corresponding to Fig. \ref{figure:lotteryticketacc}). We calculate the sparsity in every epoch and repeat five times. The black
curve represents the mean of the sparsity and shaded area shows the standard deviation of sparsity. The
vertical blue line shows the epochs that we choose to early stop. We choose the log-scale epochs for achieve larger range of sparsity. \label{figure:lotteryticketsparsity} }
\end{figure}

%

\end{document}